\documentclass{article}
\usepackage{fullpage}
\usepackage[rgb,dvipsnames]{xcolor}
\usepackage{hyperref}
\hypersetup{
    colorlinks=true,
    citecolor=ForestGreen,
    filecolor=black,
    linkcolor=VioletRed,
    urlcolor=blue
}
\usepackage{url}            
\usepackage{booktabs}       
\usepackage{amsfonts}       
\usepackage{nicefrac}       
\usepackage{microtype}      
\usepackage{lipsum}		
\usepackage{graphicx}
\usepackage{doi}
\usepackage{verbatim}      
\usepackage{subfigure}
\usepackage[inline]{enumitem}

\usepackage[linesnumbered,ruled,vlined]{algorithm2e}

\SetCommentSty{mycommfont}

\newcommand{\mb}[1]{\ensuremath{\mathbf{#1}}\xspace}

\SetKwInput{KwInput}{Input}                
\SetKwInput{KwOutput}{Output}              
\usepackage{amssymb,amsthm,amsmath}
\usepackage{thmtools}
\usepackage{thm-restate}

\usepackage{bm, bbm}
\usepackage{graphicx}

\usepackage{tcolorbox}
\newtcolorbox{mybox}{colback=red!5!white,colframe=red!75!black}

\usepackage[capitalise,nameinlink]{cleveref}

\theoremstyle{plain}
\newtheorem{theorem}{Theorem}[section]
\newtheorem{lemma}[theorem]{Lemma}

\newtheorem{fact}[theorem]{Fact}
\newtheorem{definition}[theorem]{Definition}

\theoremstyle{remark}
\newtheorem*{remark}{Remark}

\newtheorem{example}{Example}

\newcommand{\ignore}[1]{}

\usepackage{tikz}
\usetikzlibrary{positioning}

\DeclareMathOperator*\Exp{\mathbb{E}}

\DeclareMathOperator*\Var{\bf Var}
\DeclareMathOperator*\Cov{\bf Cov}

\DeclareMathOperator\erf{erf}

\newcommand{\prn}[1]{\left(#1\right)}

\newcommand{\poly}[1]{\textup{poly}\prn{#1}}

\newcommand{\ball}{\mathcal{B}}

\newcommand{\R}{\mathbb{R}}

\newcommand{\cQ}{\mathcal{Q}}
\newcommand{\cP}{\mathcal{P}}

\newcommand{\cO}{\mathcal{O}}
\newcommand{\cA}{\mathcal{A}}
\newcommand{\cK}{\mathcal{K}}
\newcommand{\cL}{\mathcal{L}}

\newcommand{\tv}{\mathrm{d_{TV}}}
\newcommand{\kl}{\mathrm{d_{KL}}}
\newcommand{\Paren}[1]{\left(#1\right)}

\newcommand{\Abs}[1]{\left\lvert#1\right\rvert}
\newcommand{\gauss}{\mathcal{N}}

\newcommand{\eps}{\varepsilon}

\usepackage{mathtools}

\newcommand{\obs}{\mathbb{S}}
\newcommand{\indic}{\bm{1}}

{
  \theoremstyle{plain}
  \newtheorem*{assumption1f}{Assumption 1.4}
}
{
  \theoremstyle{plain}
  \newtheorem*{theorem1f}{Theorem 1.5}
}

\title{Learning High-dimensional Gaussians from Censored Data}
\author{Arnab Bhattacharyya \\
The University of Warwick \\
\texttt{arnab.bhattacharyya@warwick.ac.uk}
\and Constantinos Daskalakis \\
MIT \\
\texttt{costis@mit.edu}
\and Themis Gouleakis \\
Nanyang Technological University\\
\texttt{themis.gouleakis@ntu.edu.sg}
\and 
Yuhao Wang \\
National University of Singapore\\
\texttt{yuhaowang@u.nus.edu}
}
\date{}

\begin{document}
\maketitle

\begin{abstract}
We provide efficient algorithms for the problem of distribution learning from high-dimensional Gaussian data where in each sample, some of the variable values are missing. We suppose that the variables are {\em missing not at random (MNAR)}.
The {\em missingness model}, denoted by $\obs(\mb{y})$, is the function that maps any point $\mb{y}\in \R^d$ to the subsets of its coordinates that are seen. In this work, we assume that it is known. We study the following two settings:
\begin{enumerate}[wide, labelwidth=!, labelindent=0pt]
\item[(i)] [\textbf{Self-censoring}] An observation $\mb{x}$ is generated by first sampling the true value $\mb{y}$ from a $d$-dimensional Gaussian $\gauss(\bm{\mu}^*, \mb{\Sigma}^*)$ with unknown $\bm{\mu}^*$ and $\mb{\Sigma}^*$. 
For each coordinate $i$, there exists a set $S_i\subseteq \mathbb{R}^d$ such that $x_i=y_i$ if and only if $y_i\in S_i$. Otherwise, $x_i$ is missing and  takes a generic value (e.g ``?"). 

We design an algorithm that learns $\gauss(\bm{\mu}^*, \mb{\Sigma}^*)$ up to TV distance $\eps$, using $\poly{d, 1/\eps}$ samples, assuming only that each pair of coordinates is observed with sufficiently high probability.
\item[(ii)] [\textbf{Linear thresholding}] An observation $\mb{x}$ is generated by first sampling $\mb{y}$ from a $d$-dimensional Gaussian $\gauss(\bm{\mu}^*, \mb{\Sigma})$ with unknown $\bm{\mu}^*$ and known $\mb{\Sigma}$, and then applying the missingness model $\obs$ where $\obs(\mb{y}) = \{i \in [d]: \mb{v}_i^T \mb{y} \leq b_i\}$ for some $\mb{v}_1, \dots, \mb{v}_d \in \R^d$ and $b_1, \dots, b_d \in \R$. 
We design an efficient mean estimation algorithm, assuming that none of the possible missingness patterns is very rare conditioned on the values of the observed coordinates and that any small subset of coordinates is observed with sufficiently high probability. 
\end{enumerate}

\end{abstract}

\section{INTRODUCTION} 
\label{sec:intro}

Missing data is a quite prevalent factor contributing to bias in statistical inference. It arises from various causes, such as limitations in instruments leading to unreliable data, incomplete data collection resulting in missing relevant information, societal biases influencing the suppression of observations, behavioral biases leading to subjects dropping out of studies or avoiding survey questions, ethical, legal, or privacy considerations restricting the utilization of collected data, and other similar factors.
Unfortunately training models without consideration of missing data can lead to models that incorporate biases in the training data and make incorrect predictions, which may in turn reinforce those biases when the models are deployed. 

Since the early days of statistics, missing data has been a well-known challenge in statistical inference, which occurs in a variety of domains, such as biology, physics, clinical trial design, genetics, economics, survey research, and the social sciences.
It has motivated a vast effort towards developing methodologies that are more robust to missing data. As example, we refer the reader to some of the early works in statistics ~\cite{galton1898examination,pearson1902systematic,pearson1908generalised,lee1914table,fisher1931properties},  some standard references in statistics and econometrics ~\cite{tobin1958estimation,amemiya1973regression,hausman1977social,heckman1979sample,hajivassiliou1998method,little2019statistical}, works targeting missing data in specific domains ~\cite{warga1992bond,brick1996handling,troyanskaya2001missing,armitage2008statistical,honaker2010missing},   books overviewing this literature ~\cite{maddala1986limited,breen1996regression,balakrishnan2014art}, and finallly some recent work in computer science ~\cite{mohan2013graphical,daskalakis2018efficient,daskalakis2019computationally,daskalakis2020truncated,daskalakis2021statistical,daskalakis2021efficient,kontonis2019efficient,fotakis2020efficient,plevrakis2021learning}.

The effect that data misssingness has on statistical inference depends heavily on the missingness model. In general, missingness models in which the value of some datapoint influences whether or not it will be missing from the dataset are harder to deal with compared to models in which this happens randomly. 

Techniques that have been extensively researched in scenarios where missingness either does not depend on the data or only depends on the observed data are referred to as missing completely at random (MCAR) and missing at random (MAR) respectively \cite{rubin1976inference,tsiatis2006semiparametric, little2019statistical}. In problems where missing entries depend on the underlying values which are themselves censored, known as missing not at random (MNAR), is substantially more difficult and less explored \cite{robins1997non, rotnitzky1997analysis, scharfstein1999adjusting, shpitser2015missing, adak2020classification}. The MNAR model is quite often relevant in practical applications. For example, the depression registry for mental health status is more likely to have missing questionnaires leading to the self-censoring missingness \cite{carreras2021missing}. Data are missing by design due to the limitations of measurement resources, or the treatment discontinuation when participants go off-control due to the lack of tolerability \cite{little2012prevention}.

The goal of this work is to advance our understanding of density estimation in the non-asymptotic sample regime when  data is missing not at random. In particular, we consider the standard task of high-dimensional Gaussian distribution estimation, albeit in settings where every sample of the Gaussian may have a subset of its coordinates censored and which subset this is depends on the sample itself. We consider two models for how the censoring may depend on the sample:

\begin{itemize}[leftmargin=*]
    \item Self-censoring model (see \cref{sec:self censoring}): in this model, a sample $\mb y$ is drawn from an underlying Gaussian distribution ${\cal N}({\bm{\mu}}^*, {\mb \Sigma}^*)$, and each coordinate $y_i$ of this sample is censored (i.e.~replaced with a `?') depending on whether or not it satisfies a coordinate-specific Boolean predicate, i.e.~whether $S_i(y_i)=1$ or not. 
    \item Linear thresholding model (see \cref{sec:linear masking}): in this more challenging model, whether or not each coordinate is seen depends on whether the {\bf whole} sample satisfies a coordinate-specific predicate. 
\end{itemize}

Our goal in both cases is to identify conditions on the predicates and the underlying distribution under which their parameters can be estimated computationally and sample-efficiently in each of the aforedescribed models. Our work advances prior research on Gaussian estimation in the presence of MNAR data in the non-asymptotic sample regime along the following axes:

\begin{itemize}[leftmargin=*]
    \item Gaussian estimation under censoring (see e.g.~the classical works of \cite{galton1898examination,pearson1902systematic,pearson1908generalised,fisher1931properties} and the ensuing literature): Prior work on this problem in the non-asymptotic sample regime studies the ``all-or-nothing setting,'' where either all coordinates or no-coordinate can be observed ~\cite{daskalakis2018efficient}. They also require that some absolute constant fraction of the Gaussian can be observed. In comparison to this work, we allow  heavily corrupted data where no such constant fraction exists where no coordinate is missing. However, we allow the predicate determining the censoring of coordinate $i$ to either be very general but only dependent on this coordinate (self-censoring model), or depend on all coordinates but be simpler, namely 
    a hyperplane (linear-thresholding model). 
    \item Gaussian estimation under self-selection (see e.g.~the classical work of \cite{roy1951some} and the ensuing literature): Prior work on this problem in the non-asymptotic sample regime \cite{cherapanamjeri2023makes} studies specific selection mechanisms (in particular hiding all but the maximum coordinate of each sample) and also assumes independence among the coordinates. In comparison to this work, we allow correlations among coordinates and more general masking mechanisms. However, we focus on Gaussian distributed coordinates while they can accommodate non-parametric distributions.
\end{itemize}
\ignore{
\paragraph{Problem settings}
It comes out that the distribution learning problem under the censorship model of \emph{missing not at random} captures a wide variety of settings such as \emph{self-selection model}, \emph{truncation}
\begin{itemize}[leftmargin=*]
    \item \textcolor{red}{Comparing missingness to truncation} - earlier works assume a box. For arbitrary truncation, (GTZ ’18) assumes an $\alpha$ portion to be observed over all variables. Later works use EM for non-convex likelihood function. 
    \item \textcolor{red}{Comparing missingness to self-selection}: similar to the truncated and censored linear regression problem. In regression, you explicitly assume that the last coordinates (label) are determined by the previous coordinates using a linear or polynomial or whatever mapping. 
\end{itemize}
}
\subsection{OUR CONTRIBUTIONS} 
 In this work, we are interested in recovering the ``uncorrupted'' Gaussian distribution given samples from a ``corrupted'' distribution (according to our missingness model). We use a population maximum likelihood approach as the estimation algorithm, and apply projected stochastic gradient descent on the likelihood function. We give theoretical proof of fast convergence in the parameter space. 

A {\em missingness model} is defined by a function $\obs: \R^d \to 2^{[d]}$. For an underlying $d$-dimensional vector $\mb{y}$, $\obs(\mb{y})$ is interpreted as the set of coordinates of $\mb{y}$ that are not missing. An observation is a pair $(A, \mb{x})$, where $A = \obs(\mb{y})$ and $\mb{x} = \mb{y}_A$ for an underlying sample $\mb{y} \in \R^d$.

\subsubsection{DISTRIBUTION LEARNING UNDER THE SELF-CENSORING MECHANISM} \label{sec:self censoring}
\paragraph{Self-censoring Missingness Model.}
\textcolor{black}{The Self-censoring mechanism is commonly encountered in practice. In this model the missingness of an outcome is affected by its underlying value. For example, smokers are not willing to report their smoking behavior in insurance applications. Voters holding particular beliefs may not disclose their political preferences in election surveys.}
\textcolor{black}{The Self-censoring model is of significant interest because the model is: (i) conceptually well-motivated, and  (ii) can be considered as a baseline for other more complex missingness models.}
We say that $\obs$ is a {\em self-censoring missingness model} if there exist sets $S_1, \dots, S_d$ such that $\obs(\mb{y}) = \{i \in [d]: y_i \in S_i\}$.  

Our result in this setting rests on the following hypothesis:
\begin{restatable}{assumption}{censoringmissingness}
\label{censoring_assumpt}
For any pair of coordinates $i, j \in [d]:$
$
\Pr_{\mb{y} \sim \gauss(\bm{\mu}^*, \bm{\Sigma}^*)}[y_i \in S_i, y_j \in S_j] \geq \alpha.
$
\end{restatable}

\begin{restatable}{theorem}{censoringmain}
\label{th:censoring_main}
Suppose we can observe samples from $\gauss(\bm{\mu}^*, \mb{\Sigma}^*)$ censored through a self-censoring missingness model $\obs$.
\textcolor{black}{If \cref{censoring_assumpt} is satisfied for some constant value of the parameter $\alpha$,} there exists a polynomial-time algorithm that recovers estimated $\bm{\mu}^\ast, \mb{\Sigma}^\ast$ with arbitrary accuracy. Specifically, for all $\varepsilon > 0$, 
\textcolor{black}{and given that the eigenvalues of $\mb{\Sigma}^\ast$ lie in the interval $[\lambda_{min},\lambda_{max}]$,
the algorithm uses $\tilde{\mathcal{O}}\left(\frac{d^2(\lambda_{max}/\lambda_{min})^2}{\alpha\varepsilon^2 }\right)$ samples} and produces estimates that satisfy the following:
\begin{align*}
    & \left\|\mb{\Sigma}^{^\ast -1 / 2}(\bm{\mu}^\ast-\bm{\hat{{\mu}}}) \right\|_{2} \leq \mathcal{O}( \varepsilon) ; \\
    \quad \text { and } & \left\|\bm{I}-\mb{\Sigma}^{\ast -1 / 2} \mb{\hat{\Sigma}} \mb{\Sigma}^{\ast -1 / 2}\right\|_{F} \leq \mathcal{O}(  \varepsilon).
\end{align*}
\textcolor{black}{Note that the sample complexity is proportional to $1/\alpha$.}
Furthermore, under the above conditions, we have 
$\tv(\gauss(\bm{\mu}^\ast, \mb{\Sigma}^\ast), \gauss(\bm{\hat{\mu}}, \mb{\hat{\Sigma}})) \leq \mathcal{O}(\varepsilon)$.
\end{restatable}

\subsubsection{MEAN ESTIMATION UNDER LINEAR THRESHOLDING MISSINGNESS} \label{sec:linear masking}

We say $\obs$ is a {\em linear thresholding missingness model} if there exist $\mb{v}_1, \dots, \mb{v}_d \in \R^d$ and $b_1, \dots, b_d \in \R$ such that $\obs(\mb{y}) = \{i \in [d]: \mb{v}_i^T \mb{y} \leq b_i\}$. For instance, if for a pair of coordinates $x_i, x_j$, we can only observe the 
maximum of the two, this can be modeled by a linear thresholding model where the $i$'th coordinate is observed if $x_j-x_i \leq 0$ and the $j$'th coordinate is observed if $x_i - x_j \leq 0$. 

Our main algorithmic result rests on the following two data hypotheses:
\begin{restatable}{assumption}{missingnessassumpt}
\label{missingness_assumpt} There exist some $\alpha, \beta > 0$ such that for any set $A \subseteq [d]$ of size at most $\beta d$,
\[\Pr_{\mb{y} \sim N(\bm{\mu}^*, \mb{\Sigma})}[A \subseteq \obs(\mb{{y}})]\geq \alpha.\] \textcolor{black}{Note that we only consider the case where $\beta d$ is a positive integer without loss of generality.}
\end{restatable}

\begin{restatable}{assumption}{depmiss}(Informal)
\label{dependence_assumpt}
There exists an {\em anchoring set} of coordinates $C$ such that (i) $C$ is observed in every sample, and (ii) conditioned on the values at $C$, each missingness pattern occurs with probability $0$ or at least $\gamma$.
\end{restatable}
\noindent 
The second assumption states that the anchoring subset is 
compulsorily observed and values at these coordinates determine the missingness pattern (the set of coordinates observed) almost fully. 
\textcolor{black}{This is in analogy to the anchor topic modeling used in natural language processing, which is a variation of probabilistic topic modeling that incorporates a set of predefined ``anchor words" to guide the topic modeling process. Our anchored missingness is similar to the ``anchor words" assumption.}
For instance, the anchoring subset might be a set of questions in a questionnaire that are mandatory to answer and whose values are very indicative of the respondant's behavior. 
We now informally state our main algorithmic result here:
\begin{restatable}{theorem}{linearmaskingmain}
\label{th:general_main}
For a known covariance matrix $\bm{\Sigma}$, suppose we can observe samples from $\gauss(\bm{\mu}^*, \bm{\Sigma})$ censored through a linear thresholding missingness model $\obs$. If \cref{missingness_assumpt} and \cref{dependence_assumpt} are satisfied, there exists a polynomial-time algorithm that recovers estimated $\bm{\mu}^\ast$ with arbitrary accuracy.  
Specifically, for all $\eps > 0$, \textcolor{black}{the algorithm uses $\mathrm{poly}(d,1/\alpha, 1/\beta, 1/\gamma,{\lambda_{\max}(\mb{\Sigma})}/{\lambda_{\min}(\mb{\Sigma})},$ $1/\eps,\log(1/\delta))$ samples and running time}, and with probability at least $1-\delta$, produces an estimate $\bm{\hat{\mu}}$ such that
$\Vert \bm{\hat{\mu}} - \bm{\mu^*} \Vert_\Sigma \leq \eps.$
\end{restatable}

\subsection{OUR TECHNIQUES}
\paragraph{Self Censoring}
In the self censoring model, we show that the problem can be reduced 
to solving truncation problems in each of the $2$-dimensional subspaces spanned by pairs $e_i,e_j$ of the basis vectors. This allows us to use a $2$-dimensional version of the algorithm in \cite{daskalakis2018efficient} as subroutine for our algorithm in order to extract the information about pairwise correlation of coordinates needed to reconstruct the true covariance matrix $\mb{\Sigma}^\ast$. The mean is reconstructed in a more straightforward way via solving $1$-dimensional truncation problems for each coordinate.


In particular, to estimate the diagonal entries of $\mb{\Sigma}^\ast$, we use $1$-dimensional
 subproblems and for each off-diagonal entry $\Sigma_{ij}$ we solve the $2\times 2$
 subproblem on the coordinates $i$ and $j$, but we only care about the off-diagonal entries 
 of the result (ignoring the rest) and we use exactly that value as the estimate for the entry.
To achieve the required guarantees, assuming the condition number of the true covariance matrix $\mb{\Sigma}^\ast$ is constant,  one needs to run the subproblems with $\varepsilon^\prime=c\cdot\varepsilon/d$ for some constant $c<1$
(since the Frobenious distance can be at most a factor $d$ larger than the maximum entry-wise difference), which would imply a sample complexity of $O(1/{\varepsilon^\prime}^2)$
for each subproblem. Therefore, we will need $O(1/{\varepsilon^\prime}^2)=O(d^2/\varepsilon^2)$ samples in which both coordinates $i$ and $j$ are present for each of the $O(d^2)$ possible coordinate pairs. According to our \cref{missingness_assumpt}, this happens with probability at least $\alpha$ for a particular pair of coordinates. Therefore, if we draw $O(\frac{d^2\log(1/\delta)}{\alpha\varepsilon^2})$ samples, we get that for each pair of coordinates, we have the required amount of samples with probability at least $1-\delta$. Since we need $\delta<\frac{1}{d^2}$ to apply a union bound, we have that $O(\frac{d^2\log d}{\alpha\varepsilon^2})=\Tilde{O}(\frac{d^2}{\alpha\varepsilon^2})$ samples are sufficient for \mb{\Sigma^\ast} with constant condition number.

\paragraph{Linear-thresholding Model}
For the linear thresholding model, the above reduction does not work because the problem can no longer be "decomposed" into $2$-dimensional ones. The reason is that whether or not some pair of coordinates $(x_i,x_j)$ appears can now be affected by the value of $\bm{x}$ in coordinates different than $i$ and $j$. 
Therefore, we design a projected stochastic gradient descent (PSGD) algorithm that, given the covariance matrix $\mb{\Sigma}$ is known, yet arbitrary, it will give us an estimate of the true mean $\bm{\mu}^\ast$ of the original distribution, which can be arbitrarily close to it with the right choice of parameters. The algorithm first uses an empirical estimator for the initialization of the estimate. We show that its distance to the true mean is bounded as a function of $\mb{\Sigma}$ and the parameters of \cref{missingness_assumpt}. Subsequently, we run a PSGD algorithm, whose projection step maintains this property. The gradient sampling step is non-trivial as a straightforward rejection sampling approach would run in exponential time. Therefore, we resort to a {\em Langevin Monte Carlo algorithm} which yields an approximately unbiased sample of the gradient. The projection set in this algorithm ensures that the centralized second moment of the gradient estimator is bounded, while its bias is also kept small. Combining this with our lower bound on the convexity parameter of the strongly convex likelihood function $\ell(\bm{\mu})$, we are able to show the result.  

\subsection{RELATED WORK}
\label{sec:related-work}
\paragraph{High-dimensional distribution learning}
\cite{kearns1994learnability} initiated a systematic investigation of the computational complexity of distribution learning. Since then, there has been a large volume of works devoted to the parameter and distribution learning from a wide range of distributions in both low and high dimensions \cite{dasgupta1999learning, sanjeev2001learning,
chan2013learning, ge2015learning, diakonikolas2019robust, bakshi2022robustly}.
Broadly, this problem falls into the realm of robust statistics. Following the pioneering works by \cite{tukey1960survey, huber1992robust},
other recent works on high dimensional robust distribution learning can be found at \cite{charikar2017learning, rekatsinas2017holoclean, diakonikolas2018robustly, khosravi2019expect,  diakonikolas2020robustly, kane2021robust}.
We will be particularly interested in robustly estimating mean and covariance from high-dimensional data with partially-reliable data samples \cite{baranchik1964multiple,szatrowski1980necessary, stein1981estimation, boldea2009maximum, belkin2010polynomial, pascal2013parameter,lai2016agnostic, diakonikolas2019recent, diakonikolas2019robust, lei2020fast, cheng2020high, cherapanamjeri2020list, hopkins2020robust}. 
Settings similar to ours are studied in \cite{liu2021robust, hu2021robust} regarding robust mean estimation with coordinate-level corruptions. 
In this paper, we obtain stronger guarantees for the mean estimation, yet incomparable to \cite{liu2021robust} due to their stronger corruption model.

\paragraph{Learning from truncated or censored samples}
Distribution learning under censored, truncated mechanisms has had a long history.  
Censoring happens when the events can be detected, but the measurements (the values) are completely unknown, while truncation occurs when an object falling outside some subset are not observed, and their count in proportion to the observed samples is also not known, see \cite{deemer1955estimation, cohen1957solution, dixon1960simplified,  haas1990estimation,cohen1991truncated,barr1999mean, cha2013rethinking, charikar2017learning} for an overview of the related works in estimating the censored or truncated normal or other type of distributions. 
\ignore{Specifically, \cite{pearson1902systematic, pearson1908generalised,lee1914table} used the method of moments, while \cite{fisher1931properties} used the maximum likelihood approach for the distribution learning from truncated samples. Since then, }\cite{daskalakis2018efficient, daskalakis2019computationally, daskalakis2020truncated} developed computationally and statistically efficient algorithms under the assumption that the truncation set is known. 
Furthermore, \cite{wu2019learning}
considered the problem of estimating the parameters of a $d$-dimensional rectified
Gaussian distribution from i.i.d. samples. This can be seen as a special case of the self-censoring truncation, where the truncation happens due to the ReLU generative model. \cite{shpitser2015missing} explored the identification and estimations conditions when data are missing not-at-random. While
\cite{bhattacharya2020identification, nabi2020full, malinsky2021semiparametric} 
explored the necessary and sufficient graphical conditions to recover the full data distribution under no self-censoring condition.

\paragraph{Learning from general missingness}
More broadly, self-selection models fall under the literature of regression with MNAR in the outcomes \cite{rotnitzky1995semiparametric, rotnitzky1998semiparametric, tchetgen2018discrete}. Unlike self-censoring, this project doesn't restrict the form of the representation. Two most popular methods are the expectation-maximization algorithm \cite{dempster1977maximum} and Gibbs sampling \cite{geman1984stochastic} under MAR. Despite the long history and the application of missing data models, most of the existing methods with regard to robust learning \cite{ramoni2001robust} are consistent in the asymptotic sample regime. For example, likelihood method \cite{enders2001relative}, multiple imputation \cite{allison2000multiple}, semiparametric estimation with influence function \cite{robins2000sensitivity}, inverse probability weighted complete-case estimator \cite{wooldridge2007inverse, seaman2013review}, and double/debiased machine learning \cite{chernozhukov2018double}. See textbook \cite{tchetgen2006statistical, tsiatis2006semiparametric, van2018flexible} for more introductions and further applications in this field. Recently, there are several finite sample guarantees for the double robust estimator when data are MNAR \cite{chernozhukov2018biased, chernozhukov2021simple} and high-dimensional \cite{quintas2022finite}. In addition to the works discussed, there has been significant research on detecting truncation \cite{de2023testing, de2024detecting}  and estimation under unknown truncation \cite{kontonis2019efficient, diakonikolas2024statistical}.

\section{NOTATIONS AND PRELIMINARIES}
\label{sec:notations}
Throughout, let $d\geq 1$ denote the dimension of the underlying domain. For a $d$-dimensional vector $\mb{u}$ and a subset $A  \subseteq [d]$, let $\mb{u}_A \in \R^{|A|}$ denote the restriction of $\mb{u}$ to the coordinates in $A$.
A {\em missingness model} is defined by a function $\obs: \R^d \to 2^{[d]}$. Given a distribution $\mathcal{D}$ on $\R^d$, an {\em observation of $\mathcal{D}$ censored by $\obs$} is a pair $(A, \mb{x}) \in 2^{[d]} \times \R^{|A|}$, generated by first sampling $\mb{y}\sim \mathcal{D}$ and then setting $A = \obs(\mb{y})$ and $\mb{x} = \mb{y}_A$. The interpretation is that $y_i$ is seen for every $i \in \obs(\mb{y})$ while $y_i$ is missing for every $i \not \in \obs(\mb{y})$.  We denote the resulting distribution on pairs by $\mathcal{D}^{\obs}$. If the density function of $\mathcal{D}$ is $f$, then the density function of $\mathcal{D}^\obs$ is $f^\obs$ defined as:
\begin{equation}
f^\obs(A, \mb{x}) = \int_{\mb{y} \in \R^d} \indic[\obs(\mb{y})=A]\cdot \delta(\mb{y}_A-\mb{x}) f(\mb{y}) d\mb{y}.
\end{equation}
Note that $\sum_{A \subseteq [d]}\int_{\mb{x} \in \R^{|A|}} f^\obs(A, \mb{x}) = 1$, as desired. 
We say that $\obs$ is a {\em self-censoring missingness model} if there exist sets $S_1, \dots, S_d$ such that $\obs(\mb{y}) = \{i \in [d]: y_i \in S_i\}$. We say $\obs$ is a {\em linear threshold missingness model} if there exist $\mb{v}_1, \dots, \mb{v}_d \in \R^d$ and $b_1, \dots, b_d \in \R$ such that $\obs(\mb{y}) = \{i \in [d]: \mb{v}_i^T \mb{y} \leq b_i\}$. 
\begin{restatable}{fact}{factFrobenius}\label{fact:Frobenius}
    Let $A=\{a_{ij}\}$ and $B=\{b_{ij}\}$ be two $d\times d$ matrices such that $\forall i,j: \vert  a_{ij}-b_{ij}\vert \leq \delta$. Then, $\Vert A-B \Vert_F \leq \delta \cdot d$.
\end{restatable}

\section{DISTRIBUTION LEARNING UNDER SELF-CENSORING MISSINGNESS}
\label{sec:self-censoring}
The problem of learning a distribution from truncated samples was studied in \cite{daskalakis2018efficient}. Their guarantee, as presented in \cref{cor:das}, is given under the assumption that a fraction $\alpha$ of all the samples is fully observed across all dimensions.
\ignore{
\begin{theorem}[\cite{daskalakis2018efficient}]
\label{thm:das}
Given oracle access to a measurable set $T$, whose measure under some unknown d-variate normal $\gauss(\bm{\mu}^\ast, \mb{\Sigma}^\ast)$ is at least some constant $\alpha > 0$, and samples $\mb{x}^{(1)}, \mb{x}^{(2)}, \dots, \mb{x}^{(n)}$ from $\gauss(\bm{\mu}^\ast, \mb{\Sigma}^\ast)$ that are truncated to this set, there exists a polynomial-time algorithm that recovers estimates $\hat{\bm{\mu}}$ and $\hat{\mb{\Sigma}}$. In particular, for all $\varepsilon > 0$, the algorithm uses $\tilde{\mathcal{O}}(d^2/\varepsilon^2)$ truncated samples and queries to the oracle and produces estimates that satisfy the following with probability at least 99\%.
\begin{align*}
    & \Vert \mb{\Sigma}^{^\ast -1/2} (\bm{\mu}^\ast - \bm{\hat{\mu}}) \Vert_2 \leq \varepsilon;  \\
    \quad \text{and} \quad 
    &\Vert \bm{I} - \mb{\Sigma}^{^\ast -1/2}\mb{\hat{\Sigma}}\mb{\Sigma}^{^\ast -1/2} \Vert_F \leq \varepsilon.
\end{align*}

\end{theorem}
}
\begin{theorem}[adapted from \cite{daskalakis2018efficient}]
    \label{cor:das}
Given oracle access to a measurable set $T$, whose measure under some unknown d-variate normal $\gauss(\bm{\mu}^\ast, \mb{\Sigma}^\ast)$ is at least some constant $\alpha > 0$, and samples $\mb{x}^{(1)}, \mb{x}^{(2)}, \dots, \mb{x}^{(n)}$ from $\gauss(\bm{\mu}^\ast, \mb{\Sigma}^\ast)$ that are truncated to this set, there exists a polynomial-time algorithm that recovers estimates $\hat{\bm{\mu}}$ and $\hat{\mb{\Sigma}}$. In particular, for all $\varepsilon > 0$, the algorithm uses $\tilde{\mathcal{O}}(d^2/\varepsilon^2)$ truncated samples and queries to the oracle and produces estimates that satisfy the following with probability at least 99\%.
\begin{align*}
    & \Vert (\bm{\mu}^\ast - \bm{\hat{\mu}}) \Vert_2 \leq \varepsilon\sqrt{\lambda_{max}};  
    \quad \text{and} \quad 
    \Vert  \mb{\Sigma}^{^\ast }-\mb{\hat{\Sigma}} \Vert_F \leq \varepsilon\lambda_{max}.
\end{align*} 
\end{theorem}
 This simplifies the problem because with enough samples, part of the shape of the Gaussian distribution can be observed, allowing for simultaneous estimation of the mean and covariance. In contrast, the self-censoring missingness only allows us to observe a subset of samples, making the estimation problem more challenging. The goal is to recover $\bm{\mu}^*$ and $\mb{\Sigma^*}$ under minimal assumptions on the censoring mechanism. In this section, we present and analyze our algorithm for estimating the true mean and covariance of the multivariate normal distribution under self-censoring missingness. 

\ignore{
Formally, there exist self-censoring functions $S_1, \dots, S_d : \R\to \{0,1\}$ such that for $\mb{x}\in \R^d$, the component $x_i$ can be observed if and only if\footnote{We assume here that the censoring mechanism is deterministic, but our analysis carries through also if the censoring were probabilistic. In the latter setting, $S_i : \R \to [0,1]$, and $x_i$ is observed with probability $S_i(x_i)$.}
$S_i(x_i)=1$. In other words, if the true sample is $\mb{x}$, the observation is $\tilde{\mb{x}}$ where:
\begin{align*}
 \tilde{x}_i = 
 \begin{cases}
 x_i & \text{if } S_i(x_i)=1\\
 ? & \text{ otherwise.}
 \end{cases} 
 \end{align*} 
For $n$ i.i.d. samples $\mb{x}^{(1)}, \mb{x}^{(2)}, \dots, \mb{x}^{(n)} \sim \gauss(\bm{\mu}^*, \mb{\Sigma^*})$, our observation data consists of $\tilde{\mb{x}}^{(1)}, \tilde{\mb{x}}^{(2)}, \dots, \tilde{\mb{x}}^{(n)} \in (\mathbb{R} \cup \{?\})^d$ where each observation contains some amount of missing entries. 
}
The main idea behind our algorithm for self-censoring missingness is to use the solutions to $1$-dimensional and $2$-dimensional subproblems as subroutines and subsequently combine them appropriately to obtain the solution. These subproblems are either the restriction of our problem to a single coordinate or a pair of coordinates. \cref{censoring_assumpt} guarantees the existence of sufficiently many samples for these problems and allows us to use the $1D$ and $2D$ versions of Algorithm 1 in \cite{daskalakis2018efficient} as our  \texttt{Univariate\_SGD\_truncation} and \texttt{Bivariate\_SGD\_truncation} estimator respectively.

Even though these subroutines can give us accurate estimates for each coordinate of the true mean and the correlations between pairs of sample coordinates it is unfortunately not straightforward to provide an estimate of the $d\times d$ covariance matrix satisfying our desired guarantees. We explain below how to get around this issue.  

We reconstruct the covariance matrix by only considering pairs of coordinates.
For each $i \neq j$, we apply the $2$-dimensional version of the algorithm in \cite{daskalakis2018efficient} (\texttt{Bivariate\_SGD\_truncation}) on the $i$th and $j$'th coordinates to obtain the $2\times 2$-matrix $\hat{\mb{\Sigma}}^{ij}$. 
We will show that the $d \times d$ matrix $\hat{\Sigma}$ whose off diagonal entries $(\hat{\Sigma}_{ij})$ are given by the off diagonal entries $(\hat{\Sigma}^{ij}_{12})$ of the corresponding $2\times 2$ matrix is a good approximation for the true $\Sigma$.

We are now ready to describe \cref{algo:1} demonstrating our distribution learning approach under self-censoring missingness mechanism. 

{
\IncMargin{1.2em}
\begin{algorithm}[h]
\SetKwInOut{Input}{Input}
\Input{Data $\mb{x} \in \R^{n \times d}$, where $n=\frac{1}{\alpha\eps^{\prime 2}}$}
\For{$i\gets1$ \KwTo $d$}{
$\hat{\mu}_i,\hat{\Sigma}_{ii}\gets \texttt{Uni\_SGD\_trunc}(\mb{x}_i, S_i)$\; 
}
\For{$i\gets 1$ \KwTo $d-1$}{
    \For{$j \gets i+1$ \KwTo d}{
    $\mb{\hat{\Sigma}}^{ij}_{12} \gets \texttt{Biv\_SGD\_trunc}(\mb{x}_i, \mb{x}_j, S_i \times S_j)$\;
    } 
} 
$\bm{\hat{\mu}} \gets [\hat{\mu}_1,\hat{\mu}_2,\!...,\hat{\mu}_d]$\; 
\For{$i\gets1$ \KwTo $d-1$}{
    \For{$j \gets i+1$ \KwTo d}{
    $\mb{\hat{\Sigma}}_{ij} \gets \mb{\hat{\Sigma}}^{ij}_{12} $; \qquad
    $\mb{\hat{\Sigma}}_{ji} \gets \mb{\hat{\Sigma}}^{ij}_{12} $\;
    } 
} 
\Return{ $(\bm{\hat{\mu}}, \mb{\hat{\Sigma}})$ }
\caption{[\texttt{Truncation\_PSGD}] Mean and covariance recovery algorithm with oracle access that generates samples with incomplete data.
}
\label{algo:1}
\end{algorithm}

\DecMargin{1.2em}
}

\paragraph{Mean Estimation}
\ignore{The most naive approach would be to 
simply take the empirical average of $\{y_i^{(j)}: \forall j \text{ such that } y_i^{(j)} \neq ?\}$, but this can be a bad approximation because the censoring may bias the empirical average. Instead, 
we adapt Algorithm 1 in \cite{daskalakis2018efficient} 
to the \texttt{Univariate\_SGD\_truncation} algorithm, and we estimate $\bm{\mu}^*$ by estimating each coordinate $\mu^*_i$ independently. 
Then we combine all learned results together to get our estimation of the mean value of a $d$-variate normal $\gauss(\bm{\mu}^\ast,\mb{\Sigma}^\ast)$.
}
We show in \cref{lm:mean} the finite sample bound with a consistent mean estimation up to a bounded error of $\mathcal{O}(\varepsilon)$. The proof is deferred to the appendix.

\ignore{
\begin{restatable}{lemma}{mean}
\label{lm:mean}
Let $\gauss(\bm{\mu}^\ast, \mb{\Sigma}^\ast)$ be the normal distribution with mean $\bm{\mu}^\ast$ and 
covariance matrix $\mb{\Sigma^\ast}$.  
Suppose that \cref{censoring_assumpt} holds for some constant $\alpha>0$, and the spectrum of eigenvalues of $\mb{\Sigma}^\ast$ is bounded by $\lambda_{min} \leq \lambda \leq \lambda_{max}$ and let $\bm{\hat{\mu}}$ be the estimated mean from the censored Gaussian in Line 6 of Algorithm 1. For all $\eps> 0$, using 
\textcolor{black}{$\tilde{\mathcal{O}}(d\frac{\lambda_{\max}(\mb{\Sigma}^\ast)}{\lambda_{\min}(\mb{\Sigma}^\ast)}/\varepsilon^{2})$}
samples\footnote{We note that the $\tilde{\mathcal{O}}_\alpha$ notation here hides both $\log d$ and $\log(1/\delta)$ factors.}: 
$\| \mb{\Sigma}^{^\ast -1/2}(\bm{\hat{\mu}} -\bm{\mu}^\ast) \|_2\leq \mathcal{O}(\varepsilon)$.
\end{restatable}
}

\begin{restatable}{lemma}{mean}
\label{lm:mean}
Let $\gauss(\bm{\mu}^\ast, \mb{\Sigma}^\ast)$ be the normal distribution with mean $\bm{\mu}^\ast$ and 
covariance matrix $\mb{\Sigma^\ast}$.  
Suppose that \cref{censoring_assumpt} holds for some constant $\alpha>0$, and let $\bm{\hat{\mu}}=(\hat{\mu}_1,\dots,\hat{\mu}_d)$ be the estimated mean from the censored Gaussian in Line 6 of Algorithm 1. For all $\eps> 0$, using 
$\mathcal{\tilde{O}}(\frac{d^2}{\alpha \varepsilon^{2}})$
samples\footnote{We note that the $\tilde{\mathcal{O}}_\alpha$ notation here hides both $\log d$ and $\log(1/\delta)$ factors.} we have that: 
    \[ \forall i\in [d]: \vert \mu_i^\ast - \hat{\mu}_i \vert \leq (\varepsilon/d) \sigma_{i} \leq (\varepsilon/d)\sqrt{\lambda_{max}(\mb{\Sigma})} \]
    where $\sigma_i$ denotes the standard deviation of coordinate $i$ (i.e $\sigma_i=\sqrt{\Sigma^\ast_{ii}}$, where $\Sigma^\ast_{ii}$ is the $i$-th diagonal entry of the covariance matrix $\mb{\Sigma^\ast}$).
\end{restatable}

\paragraph{Covariance Estimation}
In \cref{lm:sdp_spec} below, we show that if for each pair of coordinates we are given enough samples in which this particular pair is seen, we are able to obtain an accurate estimation of $\mb{\Sigma^\ast}$. \textcolor{black}{In particular, we will run the 2D version of the problem for each of the ${d \choose 2}$ pairs of coordinates and require that the estimate has accuracy $\varepsilon_2$. By applying \cref{cor:das} for $d=2$, and error $\delta=\frac{1}{100{d \choose 2}}$, we conclude that $\tilde{\mathcal{O}}(1/\varepsilon_2^2)$ samples are sufficient to achieve 99\% success probability via a union bound. We will show that $\eps_2$ doesn't need to be too small.} 

\begin{restatable}{lemma}{covariance}
\label{lm:sdp_spec}
Let $\bm{\hat{\Sigma}}$ be the matrix with entries $\hat{\Sigma}_{ij}=\hat{\Sigma}^{ij}_{12}$, where $\hat{\Sigma}^{ij}_{12}$ denotes the value of the off diagonal entries of the $2\times 2$ matrix $\bm{\hat{\Sigma}^{ij}}$.
By $\mb{\hat{\Sigma}}^{ij}$ we denote the estimation of a $2 \times 2$ covariance matrix that we get when we restrict the input data to coordinates $i$ and $j$.
Then the following holds: 
Using $\tilde{\mathcal{O}}(\frac{d^2}{\alpha\varepsilon^2 })$ samples to get the above estimates,
we have that: \[\| \mb{\Sigma^\ast} - \mb{\hat{\Sigma}}\|_F \leq \eps\lambda_{max}\]
where $\lambda_{max}$ is the maximum eigenvalue of $\Sigma^\ast$
\end{restatable}

Based on the above results, we summarize our main results under the self-censoring missingness mechanism in \cref{th:censoring_main}.
\censoringmain*

\ignore{ 
\begin{restatable}{theorem}{censoringmain}
\label{th:censoring_main}
We have a d-variate normal $\mathcal{N}(\bm{\mu}, \mb{\Sigma})$ with $n$ i.i.d. observed random vectors 
$\tilde{\mb{x}}^{(1)},\tilde{\mb{x}}^{(2)}, \dots, \tilde{\mb{x}}^{(n)} \in (\mathbb{R} \cup \{?\})^d$ by the data generator $\mathcal{O}$. 
Let $S_1, \dots, S_d : \R\to \{0,1\}$ represent the self-censoring functions such that for $\mb{x}\in \R^d$, the component $x_i$ can be observed if and only if $S_i(x_i)=1$.
There exists a polynomial-time algorithm that recovers estimated $\bm{\mu}, \mb{\Sigma}$ with arbitrary accuracy. Specifically, for all $\varepsilon > 0$, the algorithm uses $\Omega(\frac{\log(1/\delta)\cdot \log(nd /\alpha \delta) }{ \varepsilon^2 d})$ samples and produce estimates that satisfy the following:
\begin{equation}
\left\|\mb{\Sigma}^{-1 / 2}(\bm{\mu}-\bm{\hat{{\mu}}}) \right\|_{2} \leq \mathcal{O}(\sqrt{d} \varepsilon) ; \quad \text { and } \quad \left\|\bm{I}-\mb{\tilde{\Sigma}}^{-1 / 2} \mb{\Sigma} \mb{\tilde{\Sigma}}^{-1 / 2}\right\|_{F} \leq \mathcal{O}( \sqrt{d} \varepsilon)
\end{equation}
Furthermore, under the above conditions, we have the total variation distance 
$d_{TV}(\gauss(\bm{\mu}, \mb{\Sigma}), \gauss(\bm{\tilde{\mu}}, \mb{\tilde{\Sigma}})) \leq \mathcal{O}(\sqrt{d}\varepsilon)$
\end{restatable}
}

With the following lemma we will show a lower bound, which shows that even if the TV distance is large, in which case the distributions are easily distinguishable in the classical sampling model, the censoring model makes the distribution hard to distinguish.

\begin{restatable}{lemma}{selfcensoringlb}
Given $m=o(1/\sqrt{\lambda_{min}})$
censored samples according to the missingness model $\obs$ and $\eps=\Omega(\sqrt{\lambda_{min}})$. No algorithm can estimate the true mean with accuracy $O(\eps)$ and probability larger than $2/3$.
\end{restatable}
Note that, for $\eps=\Omega(\sqrt{\lambda_{min}})$ the TV distance between the distributions $P_\lambda$ and $Q_\lambda$ is $\Omega(1)$, yet the $\Omega(1/\sqrt{\lambda_{min}})$ censored samples are necessary.


\section{MEAN ESTIMATION UNDER LINEAR THRESHOLDING MISSINGNESS}
\label{sec:convex-masking}

In this section, we present sufficient conditions for mean estimation under linear thresholding missingness. As earlier, we let $\obs$ denote the missingness model, and $\gauss(\bm{\mu}^*, \bm{\Sigma})$ denote the ground truth distribution. 
Our observations are drawn from $\gauss(\bm{\mu}^*, \bm{\Sigma})^\obs$. 

We will make the following two assumptions on the missingness mechanism and the ground truth distribution.
Our first assumption ensures that any small subset of coordinates is observed simultaneously with non-negligible probability.
\missingnessassumpt*
\noindent
This is a stronger version of \cref{censoring_assumpt}.  Our second assumption postulates existence of an ``anchoring'' subset.
\begin{definition}[Anchored missingness]
\label{def:anchor}
A subset $C \subseteq [d]$ is {\em $\gamma$-anchoring} if 
\begin{enumerate}
    \item[(i)] $ C \subseteq \obs(\mb{y})$ for any  $\mb{y}$, and 
    \item[(ii)] for any $A \subseteq [d]$, $\Pr\limits_{\mb{y}\sim \gauss(\bm{\mu}^*, \bm{\Sigma})}[\obs(\mb{y}) = A \mid \mb{y}_C]$ is either 0 or at least $\gamma$.
\end{enumerate}
\end{definition}

\begin{assumption1f}\label{ass:anchor}
There exists a $\gamma$-anchoring subset $C$ for the true distribution $\gauss(\bm{\mu}^*, \bm{\Sigma})$ in combination with the missingness model $\obs$.
\end{assumption1f}

Given the assumptions above, we will prove the following result showing that we can accurately and efficiently recover the mean of the distribution using censored samples:
\ignore{
\begin{theorem1f}
Given a linear thresholding missingness model $\obs$ and a positive semidefinite matrix $\bm{\Sigma}$, suppose $\obs$ and $\gauss(\bm{\mu}^*, \bm{\Sigma})$ satisfy \cref{missingness_assumpt} and \ref{dependence_assumpt}. Then, for all $\eps>0$, there exists an algorithm using
\textcolor{black}{$\mathrm{poly}(d, 1/\alpha, 1/\beta, 1/\gamma, 1/\lambda_{\min}(\bm{\Sigma}), \lambda_{\max}(\bm{\Sigma}), 1/\eps, \log(1/\delta))$ samples and running time}, and with probability at least $1-\delta$, produces an estimate $\widehat{\bm{\mu}}$ such that $\|\widehat{\bm{\mu}}-\bm{\mu}\|_\Sigma \leq \eps$. 
\end{theorem1f}
}
\linearmaskingmain*

\paragraph{General outline} In this section, we present and analyze our mean estimation algorithm of \texttt{MissingDescent} under anchor missingness models. As a high-level overview, our approach involves running a Projected Stochastic Gradient Descent (PSGD) algorithm on a negative log-likelihood function whose optimal value coincides with the true mean. The steps of proof are as follows:
\begin{itemize}
    \item We develop an appropriate objective function in \cref{subsec:log-likelihood} and design an efficient mean estimation algorithm \texttt{MissingDescent} in \cref{algo:psgd}, assuming that any small subset of coordinates is observed with sufficiently high probability (\cref{missingness_assumpt}), and the observed missingness pattern is not very rare conditioned on the values of the observed coordinates (\cref{dependence_assumpt}).
    \item We show that our objective function is strongly convex with respect to the correct parameterization and hence the optimum is unique. Furthermore, it is equal to the true mean.
    \item 
    We analyze our \texttt{MissingDescent} algorithm in \cref{sec:gen_anal} by showing that \cref{algo:psgd} approximately optimizes $\ell$ with bounds on the runtime and sample complexity.
    \item Specifically, we show in \cref{algo:init} in \cref{subsec:algo} that we can use the \texttt{Initialize} algorithm to efficiently compute an initial feasible point to start the optimization.
    \item 
    In the \texttt{SampleGradient} algorithm in \cref{algo:sampgrad}, we demonstrate that it is possible to obtain an estimate of $\Delta \ell(\bm{\mu})$ that is approximately unbiased by sampling from the conditional distribution. Additionally, we use the \texttt{ProjectToDomain} algorithm in \cref{algo:projettodomain} to project a current guess back onto the domain.
\end{itemize}

\subsection{NEGATIVE LOG-LIKELIHOOD OBJECTIVE FUNCTION WITH ANCHOR MISSINGNESS}
\label{subsec:log-likelihood}
We will approach the mean estimation problem via optimization of the population log-likelihood with respect to a given parameter estimate $\mb{\mu}$ for the true mean $\mb{\mu}^*$. 
Define $g_{\bm{\mu}}$ to be the density function of $\gauss(\bm{\mu}, \bm{\Sigma})$:
\[
g_{\bm{\mu}}(\mb{y}) = (2\pi)^{\frac{-d}{2}} |\bm{\Sigma}|^{\frac{-1}{2}} \exp\left(-\frac{(\mb{y}-\bm{\mu})^T\bm{\Sigma}^{-1}(\mb{y}-\bm{\mu})}{2}\right).
\]
Recall the notation $g_{\bm{\mu}}^\obs$ defined in \cref{sec:notations} to be the density function of $\gauss(\bm{\mu},\bm{\Sigma})$ censored by $\obs$. We can then write down the population negative log-likelihood $\ell$ as:
\begin{align*}
\hspace{-0.3cm}
&\ell(\bm{\mu}) = \Exp_{(A,\mb{x}) \sim \gauss(\bm{\mu}^*, \bm{\Sigma})^\obs}[-\log g_{\bm{\mu}}^\obs(A, \mb{x})]\\
&=\Exp_{(A,\mb{x})}\left[- \log \int_{\mb{y}}\indic[\obs(\mb{y}) = A] \cdot \delta(\mb{y}_A-\mb{x}) \cdot  g_{\bm{\mu}}(\mb{y}) d\mb{y}\right].
\end{align*} 
In the second equality, and everywhere later, $(A, \mb{x})$ is an observation sampled from the censored version of the true distribution: $\gauss(\bm{\mu}^*, \Sigma)^\obs$.
The integral above marginalizes over all $\mb{y}$ for which the missingness model would yield the observation $(A, \mb{x})$.

The gradient with respect to $\bm{\mu}$ of $\bm{\nabla}\ell(\bm{\mu})$ can be expressed as
\small
\begin{align}
& \Exp_{(A, \mb{x})}\left[ -\frac{\int_{\mb{y}}  \bm{\Sigma}^{-1} (\mb{y}-\bm{\mu})\cdot \indic[\obs(\mb{y}) = A] \cdot \delta(\mb{y}_A-\mb{x}) g_{\bm{\mu}}(\mb{y})  d\mb{y} }{\int_{\mb{y}}\indic[\obs(\mb{y}) = A] \cdot \delta(\mb{y}_A-\mb{x}) \cdot  g_{\bm{\mu}}(\mb{y}) d\mb{y}} \right] \label{eq:general_grad}\\
&=
-\Exp_{(A,\mb{x})}[\Exp_{\mb{y} \sim \gauss(\bm{\mu}, \Sigma)}[\bm{\Sigma}^{-1}(\mb{y}-\bm{\mu}) \mid \obs(\mb{y})=A, \mb{y}_A = \mb{x}]] \label{eq:general_grad_exp}
\end{align}
\normalsize

\begin{restatable}{lemma}{globalmin}
\label{lem:globalmin}
For any $\bm{\mu} \in \R^d$, it holds that: $\ell(\bm{\mu}) \geq \ell(\bm{\mu}^*)$.
\end{restatable}

\begin{lemma}[Strong Convexity with Missing Entries]
\label{lm:strong convex general}
Given our missingness model and \cref{missingness_assumpt} \textcolor{black}{with $\beta=\frac{c}{d}$ for some integer $c\in\{1,\dots d\}$}, we have that the function with general covariance  $\ell(\bm{\mu})$ is $\lambda$-strongly convex for $\lambda=\alpha\beta/\lambda_{\max}(\bm{\Sigma})$.  
\end{lemma}
\begin{remark}
Convexity may not hold if the missignnness pattern is not linear thresholding. For example, even for $d=1$, if $\mu^*=0$ and $\obs(y) = \{1\}$ if $y \in [2,4]$ and $\emptyset$ otherwise, the function $\ell(\mu)$ is not convex.
\end{remark}
\subsection{ALGORITHM}
\label{subsec:algo}

\paragraph{Initialization}
\label{subsec:initialization}
Our first step for efficiently optimizing the negative log-likelihood function is finding a good initial point for the PSGT. Specifically, we take the empirical mean $\mb{\hat{\mu}}$.
This is a biased estimate, but we show below that this is good enough for initialization: the distance of the empirical estimates and true mean $\mb{\mu}^*$ is a constant that depends only on the constant $\beta$, mass $\alpha$ and $\lambda_{max}$ of the known $\mb{\Sigma}$. 
The pseudocode for \texttt{Initialize} appears in \cref{algo:init}. 

\begin{algorithm}[htbp]
\SetKwInOut{Input}{Input}
\Input{Access to data generator $\mathcal{O}$,  parameter \textcolor{black}{$\beta =\frac{c}{d}$ for some integer $c\in \{1, \dots, d\}$}, number of samples $M_{\text{init}}$}

$\mb{w} \gets $ empty array of length $d$\\
$\mb{X} \gets$ matrix with $M_{\text{init}}$ rows, each an independent sample from $\mathcal{O}$\\
\For{$i\gets0$ \KwTo $\lceil 1/\beta \rceil\textcolor{black}{-1}$}{
$s\gets i\beta d+1 $\\
$t\gets \textcolor{black}{\min\{ (i+1)\beta d,d\}}$\\
$\mb{Y}_i \gets $ submatrix of $\mb{X}$ consisting of columns $s, s+1, \dots, t$\\
Remove all rows of $\mb{Y}_i$ containing at least one $*$\\
$\bm{\hat{\mu}}_i \gets $ average of the rows of $\mb{Y}_{\textcolor{black}{i}}$\\
$\mb{w}[s, s+1, \dots, t] \gets \bm{\hat{\mu}}_i$
}
\Return $\mb{w}$

\caption{[\texttt{Initialize}] Initialization for the main algorithm.}
\label{algo:init}
\end{algorithm}

By \cref{missingness_assumpt}, we have that after line 7 in \texttt{Initialize}, each $\mb{Y}_i$ is the truncation of a $\beta d$-dimensional gaussian where the truncation set has mass at least $\alpha$. 
Using Lemma 6 of \cite{daskalakis2018efficient}, the mean of such a truncated gaussian is $O(\sqrt{\log(1/\alpha)})$ distance away from the untruncated mean. 
Hence, we have
$\|\Exp[\mb{w}]-\bm{\mu}^*\|_2^2 = \sum_i \|\Exp[\bm{\hat{\mu}}_i] - \bm{\mu}[i\beta d + 1, \dots, (i+1)\beta d]\|_2^2 
\leq \lambda_{\max}\sum_i \|\Exp[\bm{\hat{\mu}}_i] - \bm{\mu}[i\beta d + 1, \dots, (i+1)\beta d]\|_{\bm{\Sigma}}^2 
\leq \mathcal{O}(\frac{\lambda_{\max}}{\beta}{\log(1/\alpha)})$.\footnote{Define $\bm{x}= \Exp[\bm{\hat{\mu}}_i] - \bm{\mu}[i\beta d + 1, \dots, (i+1)\beta d]$, and the eigenvalue decomposition of $\mb{\Sigma}^{-1}$ as $\mb{\Sigma}^{-1}=Q^{\top}D^{-1}Q$. The first inequality holds because $\Vert \bm{x}\Vert_{\bm{\Sigma}}^2 = \Vert \bm{x}^{\top} Q^{\top} D^{-1} Q \bm{x} \Vert_2 =  \|D^{-1/2}Q\bm{x}\|_2^2 \geq \frac{1}{\lambda_{\max}} \Vert Q\bm{x} \Vert_2^2 = \frac{1}{\lambda_{\max}} \Vert \bm{x}\Vert_2^2$. Therefore, we have $\Vert \bm{x} \Vert_2^2 \leq \lambda_{\max} \Vert \bm{x} \Vert_{\bm{\Sigma}}^2$}
Therefore, $\|\Exp[\mb{w}]-\bm{\mu}^*\|_2 \leq \mathcal{O}({\sqrt{\frac{\lambda_{\max}}{\beta}\log(1/\alpha)}})$.
Later in \cref{sec:gen_anal}, we analyze the number of samples $M_{\text{init}}$ required for $\|\mb{w}-\bm{\mu}^*\|_2$ to satisfy this bound with high probability.

{
\begin{algorithm}[htbp]
\SetKwInOut{Input}{Input}
\SetKwFunction{Initialize}{Initialize}
\SetKwFunction{SampleGradient}{SampleGradient}
\SetKwFunction{ProjectToDomain}{ProjectToDomain}
\Input{Access to data generator $\mathcal{O}$, parameters $\beta, \lambda_{\text{sgd}}, \eta_{\text{lmc}}$, $R_{\text{lmc}}, r_{\text{proj}}$, $M_{\text{init}}$, $M_{\text{sgd}}$, $M_{\text{grad}}$}
$\bm{\mu}^{(0)} \gets \Initialize(\mathcal{O}, \beta, M_{\text{init}})$ \\
\For{$i\gets 1$ \KwTo $M_{\text{sgd}}$}{
Sample $(A^{(i)},\mb{x}^{(i)})$ from $\mathcal{O}$\\
$\eta_i \gets \frac{1}{\lambda_{\text{sgd}} \cdot i}$\\
$\mb{g}^{(i)} \gets \SampleGradient((A^{(i)},\mb{{x}}^{(i)}),\bm{\mu}^{(i-1)}, \eta_{\text{lmc}}, R_{\text{lmc}}, M_{\text{grad}})$\\
$\mb{v}^{(i)} \gets \bm{\mu}^{(i - 1)} - \eta_i \mb{g}^{(i)}$\\
$\bm{\mu}^{(i)} \gets \ProjectToDomain(\bm{\mu}^{(0)}, \mb{v}^{(i)}, r_{\text{proj}})$
}
$\bm{\bar{\mu}}\gets \frac{1}{M_{\text{sgd}}}\sum_{i=1}^{M_{\text{sgd}}} \bm{\mu}^{(i)}$\\
\Return $\bm{\bar{\mu}}$
\caption{[\texttt{MissingDescent}] Mean recovery algorithm given access to an oracle that generates samples with incomplete data.}
\label{algo:psgd}
\end{algorithm}
\DecMargin{1.2em}
}

Note that, in each iteration of SGD in \texttt{MissingDescent} (\cref{algo:psgd}), 
we choose a projection set, to make sure that PSGD converges. Specifically, we project a current guess back to a $\ball_{\bm{\Sigma}}$ ball scaled by $r_{proj}$ and centered at $\bm{\mu}^{(0)}$ as shown below:
\begin{algorithm}
\SetKwInOut{Input}{Input}
\Input{$\bm{\mu}^{(0)}, \mb{v}$, parameter $r_{\text{proj}}$}
\Return 
$\bm{\mu}^{(0)}+\min\{r_{proj},\Vert (\mb{v}-\bm{\mu}^{(0)})\Vert_{\Sigma} \}\cdot\frac{ (\mb{v}-\bm{\mu}^{(0)})}{(\Vert \mb{v}-\bm{\mu}^{(0))}\Vert_{\Sigma}}$

\caption{[\texttt{ProjectToDomain}] The function that projects a current guess back to the domain 
onto the $\ball_{\bm{\Sigma}}$ ball.}
\label{algo:projettodomain}
\end{algorithm}

Our goal is to minimize the population negative log-likelihood $\ell$ via (projected) stochastic gradient descent while maintaining its strong-convexity. 
Specifically, \cref{algo:psgd} above describes this strategy. In order to apply  \cref{algo:psgd} to our log-likelihood objective function, we need to solve the following three algorithmic problems:
\begin{itemize}[leftmargin=*]
    \item \textbf{Initialization}: efficiently compute an initial feasible point from which to start the optimization. The pseudocode for \texttt{Initialize} appears in \cref{algo:init};
    \item \textbf{Gradient estimation}: design a nearly unbiased sampler for $\bm{\nabla}\ell(\bm{\mu})$ using Langevin sampling. The \texttt{SampleGradient} pseudocode appears in \cref{algo:sampgrad}; 
    \item \textbf{Efficient projection}: perform an efficient projection into a set of feasible points to make sure that PSGD converges. The pseudocode presents in \cref{algo:projettodomain}. 
\end{itemize}

{
\IncMargin{1.2em}
\begin{algorithm}
\SetKwInOut{Input}{Input}
\Input{$(A, \mb{x})$, $\bm{\mu}$, parameters $\eta, R, M$}
$a \gets |A|$\\
Compute $\bm{\mu}_{\text{cond}}$ and $\bm{\Sigma}_{\text{cond}}$ as in (\ref{eq:newg}) and (\ref{eq:newc})\\
Let $W$ be such that $\bm{\Sigma}_{\text{cond}} = WW^\top$\\
Compute $\cL = (W^{-1} \cK)\cap \ball_{\textcolor{black}{\Sigma}}(W^{-1}\bm{\mu}_{\text{cond}},R)$ \\
$\mb{z}^{(0)} \gets $ a point in $\cL$\\
\For{$t = 0~\KwTo~ M-1$}{
Sample $\bm{\zeta}^{(t)}$ from $\gauss(0, I_{d-a})$\\
$\mb{z}^{(t+1)} \gets \Pi_\cL\left(\mb{z}^{(t)} - {\eta}(\mb{z}^{(t)} - W^{-1}\bm{\mu}_{\text{cond}}) + \sqrt{\eta} \cdot \bm{\zeta}^{(t)}\right)$
}
\Return $-\bm{\Sigma}^{-1}(\mb{x} \circ (W\mb{z}^{(M)})-\bm{\mu})$ 
\caption{[\texttt{SampleGradient}] Sampler for $\nabla \ell(\bm{\mu})$.}
\label{algo:sampgrad}
\end{algorithm}
\DecMargin{1.2em}
}

\subsection{ANALYSIS OF \texttt{MissingDescent}}\label{sec:gen_anal}

We show in this section that \cref{algo:psgd} approximately optimizes $\ell$ with bounds on the runtime and sample complexity. The following lemma describes the ingredients necessary to obtain  such bounds:

\begin{lemma}[Lemma 6 in \cite{cherapanamjeri2022makes}]
\label{lem:genmain}
Let $f: \R^k \to \R$ be a convex function, $K \subseteq \R^k$ a convex set, and fix an initial estimate $\mb{x}^{(0)} \in K$. Now, let $\mb{x}^{(1)}, \dots, \mb{x}^{(T)}$ be the iterates generated by running $T$ steps of projected SGD using gradient estimates $\mb{g}^{(1)}, \dots, \mb{g}^{(T)}$ satisfying $\Exp[\mb{g}^{(i)} \mid \mb{x}^{(i-1)}] = \nabla f(\mb{x}^{(i-1)}) + \mb{b}^{(i)}$. Let $\mb{x}_* = \arg \min_{\mb{x} \in K} f(\mb{x})$ be a minimizer of $f$. Then, if we assume:
\begin{enumerate}
\item[(i)] \textbf{Bounded step variance}: $\Exp[\|\mb{g}^{(i)}\|_2^2] \leq \rho^2$, 
\item[(ii)] \textbf{Strong convexity}: $f$ is $\lambda$-strongly convex, and 
\item[(iii)] \textbf{Bounded gradient bias}: $\|\mb{b}^{(i)}\|^2 \leq \frac{\rho^2}{2\lambda \cdot \text{diam}(K)\cdot i}$,
\end{enumerate}
then the average iterate $\widehat{\mb{x}} = \frac1T \sum_{t=1}^T \mb{x}^{(T)}$ satisfies $\Exp[f(\widehat{\mb{x}}) - f(\mb{x}_*)] \leq \frac{\rho^2}{\lambda T} (1 + \log(T))$.  
\end{lemma}

We study each of the three conditions in \cref{lem:genmain} above, before wrapping up with the overall analysis.

\subsubsection{STRONG CONVEXITY}
To show convergence of stochastic gradient descent on $\ell$, we require \emph{strong convexity} such that the optimum of $\ell$ is unique. 
Specifically, we need to show: $\bm{\nabla}^2\ell(\bm{\mu}) \succeq \beta I$ for some parameter $\beta > 0$ such that the probability mass is at least a constant. Once we proved the strong convexity, we can apply projected stochastic gradient descent (PSGD) to recover the parameter $\mu$.

\begin{restatable}[Strong Convexity with Missing Entries]{lemma}{StrongConvexityMissingEntries}
\label{lm:strong convex general}
Given our missingness model and \cref{missingness_assumpt} \textcolor{black}{with $\beta=\frac{c}{d}$ for some integer $c\in\{1,\dots d\}$}, we have that the function with general covariance $\ell(\bm{\mu})$ is $\lambda$-strongly convex for $\lambda=\alpha\beta/\lambda_{\max}(\bm{\Sigma})$.
\end{restatable}

\ignore{
\subsubsection{Bound on the gradient of the log likelihood function}

We show in \cref{lm:strong convex general} that $ \bm{\nabla}^2 \ell(\bm{\mu})$ is strongly convex and thus positive semidefinite, so its norm is the largest eigenvalue of the hessian. This eigenvalue is at most $\lambda_{max}(\bm{\Sigma}^{-1})=1/ \lambda_{min}(\bm{\Sigma})$ due to \eqref{hessian_l}. 
We now give a bound on the magnitude of $\Vert\bm{\nabla}\ell(\bm{\mu})\Vert_2$ in a region around $\bm{\mu}^\ast$. 
Since the gradient of the log likelihood is always $\lambda_{max} (\bm{\nabla}^2 \ell(\bm{\mu}))- $ Lipschitz, 
we conclude that:
\begin{align*}
\Vert\bm{\nabla}\ell(\bm{\mu})-\bm{\nabla}\ell(\bm{\mu}^*)\Vert_2 \leq \Vert\bm{\mu}-\bm{\mu}^*\Vert_2\cdot \frac{1}{\lambda_{min}(\bm{\Sigma})}
\end{align*}
Therefore, $\Vert\bm{\nabla}\ell(\bm{\mu})\Vert_2 \leq \Vert\bm{\mu}-\bm{\mu}^*\Vert_2\frac{1}{\lambda_{min}(\bm{\Sigma})}$.
}
\subsubsection{BOUNDED STEP VARIANCE AND GRADIENT BIAS}
\label{sec:bounded_variance_and_bias}

\ignore{
We now show how to bound the expectation of the squared $\ell_2$ norm of our gradient estimator. We do that by observing that the output of Algorithm \cref{algo:sampgrad} 
is contained in $\mathcal{B}_{\bm{\Sigma}}(\bm{\mu},R)$, where $R=\sqrt{d}+\sqrt{2\log(1/\delta)}$ for an appropriate value of $\delta$ and $\mathcal{B}_{\bm{\Sigma}}(\bm{\mu},R)$ denotes the set $\{\bm{x}\in \mathbb{R}^d: \Vert \bm{x}-\bm{\mu}\Vert_{\bm{\Sigma}}\leq R\}$ of points with Mahalanobis distance from $\bm{\mu}$ at most $R$. We note that $\Pr_{X\sim \gaus\cA(\bm{\mu},\bm{\Sigma})}[X\not\in\mathcal{B}_{\bm{\Sigma}}(\bm{\mu},R)]\leq \delta$. Combining this with Lemma \ref{lm:boundedvariance} below, we conclude that $\Exp[\|\mb{g}^{(i)}\|_2^2] = O(poly(d,1/\delta),...) $
}

In this section, we analyze \cref{algo:sampgrad}, \texttt{SampleGradient} with an illustration of the relationship between convex sets appeared in this section is available in \cref{fig:convex_sets}.
We first study the distribution of $\mb{z}^{(M)}$, which then will allow us to show an additive approximation guarantee for the output of the algorithm.

\begin{restatable}{theorem}{langevin}
\label{Thm:langevin}
Assume $\|\bm{\mu}^* - \bm{\mu}\|_{\bm{\Sigma}} \leq S$. For $R = \tilde{O}(\sqrt{d} + S + \log(1/\gamma\eps))$, if $M = \poly{d, S, 1/\gamma, 1/\eps}$ and $\eta=\tilde{\Theta}(R^2/M)$,
then 
\[\tv(\mb{z}^{(M)}, \gauss(W^{-1}\bm{\mu}_{\text{cond}}, I)) \leq \eps.\]
\end{restatable}


Fix $(A,\mb{x})$. Without loss of generality, assume $\bm{\mu}^* = \bm{0}$.


\begin{restatable}{corollary}{boundedvariance}
\label{lm:boundedvariance}
Let $\hat{\mb{g}}$ be the output of Algorithm \ref{algo:sampgrad} with inputs $\tilde{\mb{x}}$ and $\bm{\mu}$ and parameters $R, M, \eta$ as in \cref{Thm:langevin}. Also, let $\mb{g} = -\Exp_{\mb{y} \sim \gauss(\bm{\mu},\bm{\Sigma})}[\bm{\Sigma}^{-1}(\mb{y}-\bm{\mu}) \mid \obs(\mb{y})=A, \mb{y}_A=\mb{x}]$. Then, we have that:
\begin{equation} \label{expectation_approx}
 \Vert \Exp[\hat{\mb{g}}] -\mb{g}\Vert_2
 \leq \eps \cdot \poly{S, d, 1/\gamma, 1/\eps, \lambda_{\max}, 1/\lambda_{\min}}
\end{equation}
Furthermore, we have the following bound
\begin{equation}\label{variance_approx}
\Exp[\Vert\hat{\mb{g}}\Vert_2^2]\leq \poly{d,1/\gamma,S,{1}/{\lambda_{\min}}}
\end{equation}

\end{restatable}

\section{DISCUSSION AND FUTURE WORK}
In the context of linear-thresholding missingness with a known covariance matrix, we can obtain the mean by initially observing that the set ${\mb{y}: \obs(\mb{y})=A\wedge\mb{y}_A=\mb{x}}$ is convex for any set $A$ and any $\mb{x} \in \R^{|A|}$. By leveraging the fact (Corollary 2.1 of \cite{KP77}) that the variance of a Gaussian decreases when conditioned on a convex set, we can establish that the Hessian of our likelihood function is positive definite. This property ensures that our objective function is strongly convex and thus we can learn the distribution from a MNAR model. However, in scenarios where the covariance matrix is unknown, recovering the distribution becomes much more challenging as the Hessian of our likelihood function incorporates a fourth moment. Thus, we leave this as our future work.

\bibliographystyle{plainnat}
\newcommand{\etalchar}[1]{$^{#1}$}

\appendix
\thispagestyle{empty}
\onecolumn 

\section{DEFERRED PROOFS FROM \cref{sec:notations}}
\label{sec:deffer_proof}

Below we provide the proof of Fact \cref{fact:Frobenius} for completeness.

\factFrobenius*
\begin{proof}
    By definition of the Frobenious norm we have:
    \[
    \Vert A-B \Vert_F^2=\sum_{i=1}^d\sum_{j=1}^d (a_{ij}-b_{ij})^2\leq d^2\delta^2
    \]
    Thus, 
    \[
     \Vert A-B \Vert_F\leq \delta\cdot d
    \]
\end{proof}

\section{DEFERRED PROOFS FROM \cref{sec:self-censoring}}
\label{sec:deffer_proof}

This section provides the formal proofs that were deferred in favor of readability.
For convenience, we will restate the statements before proving them.

\mean*
\begin{proof}
    Fix a coordinate $i\in[d]$. If coordinate $i$ appears in a censored sample, then the value would follow the distribution $\gauss(\mu^\ast_i,\Sigma^\ast_{ii})=\gauss(\mu^\ast_i,\sigma_{i}^2)$. In order to apply \cref{cor:das} for $d=1$ and $\varepsilon^\prime=\varepsilon/d$, we need coordinate $i$ to be present in at least $\tilde{O}(1/\varepsilon^2)$ censored samples for every $i\in [d]$. \cref{censoring_assumpt} implies that coordinate $i$ is present in each censored sample with probability at least $\alpha$. Since in every batch of $\cO(1/\alpha(\varepsilon^\prime)^2)$ samples, there is a constant probability that the required number of $\cO(1/(\varepsilon^\prime)^2)$ appearances of coordinate $i$ is met, the error probability can be reduced to $1/d^2$ a the cost of an extra log factor in the sample complexity. Therefore, by union bound over all $d$ coordinates, the statement holds with probability at least $1-1/d$ using $\cO(\log d/\alpha(\varepsilon^\prime)^2)=\mathcal{\tilde{O}}(\frac{d^2}{\alpha \varepsilon^{2}})$ samples.      
\end{proof}

\covariance*

\begin{proof}
Consider any pair of coordinates $i,j\in [d]$. Given the sample size of $\frac{1}{\alpha\varepsilon^{\prime 2}}$, \cref{censoring_assumpt} and \cref{fact:Frobenius}, there will be at least $\frac{1}{\eps^{\prime 2}}$ samples with non-censored entries in both coordinates $i,j$ for any such pair.  
Therefore, we can apply \cref{cor:das} for $d=2$ and $\varepsilon^\prime=\varepsilon/d$ to get that: \[\Vert \mb{\Sigma}^{ \ast (ij)}-\mb{\hat{\Sigma}^{(ij)}}\Vert_F \leq \varepsilon^\prime\lambda_{max}\left(\mb{\Sigma}^{ \ast (ij)}\right) = \varepsilon\lambda_{max}\left(\mb{\Sigma}^{ \ast (ij)}\right)/d\leq \varepsilon\lambda_{max}\left(\mb{\Sigma}^{ \ast}\right)/d \] 
Therefore, all the entries of the $2 \times 2$ matrix on the lhs have absolute value at most $\varepsilon/d$ and thus this is also an upper bound on the maximum difference of corresponding off diagonal entries of the $d \times d$ matrices $\mb{\Sigma^\ast}$ and $\mb{\hat{\Sigma}}$ as constructed by Algorithm \ref{algo:1} (see line $9$). The same upper bound holds for the diagonal entries, since they are more accurately estimated using 1d subproblems (line $2$ of Algorithm \ref{algo:1}).
Therefore, the maximum entry-wise difference overall between the $d \times d$ matrices $\mb{\Sigma^\ast}$ and $\mb{\hat{\Sigma}}$ as constructed by Algorithm \ref{algo:1} is $\eps\lambda_{max}/d$. We can now apply
\cref{fact:Frobenius} to get $\| \mb{\Sigma^\ast} - \mb{\hat{\Sigma}}\|_F \leq \eps\lambda_{max}$.
\end{proof}

\paragraph{Lower bound on the sample complexity}
We now explain why some dependence on the eigenvalues of $\mb{\Sigma^\ast}$ is necessary.
\noindent To see this, consider the case where  $\lambda_{\min}(\mb{\Sigma}^\ast)=0$. In this case, all the samples come from a subspace of dimension at most $d-1$. Consider one such subspace and its translation, by an infinitesimal amount $\varepsilon$, along the eigenvector corresponding to the $0$ eigenvalue. It is not hard to see that if we take two such parallel subspaces arbitrarily close to each other, then with high probability no finite amount of censored samples can distinguish between these 2 cases, which would be required for keeping this Mahalanobis error bounded.

An illustrative example is in \cref{fig:2d_censoring}:
\begin{figure}
    \centering
    \includegraphics[width=0.45\textwidth]{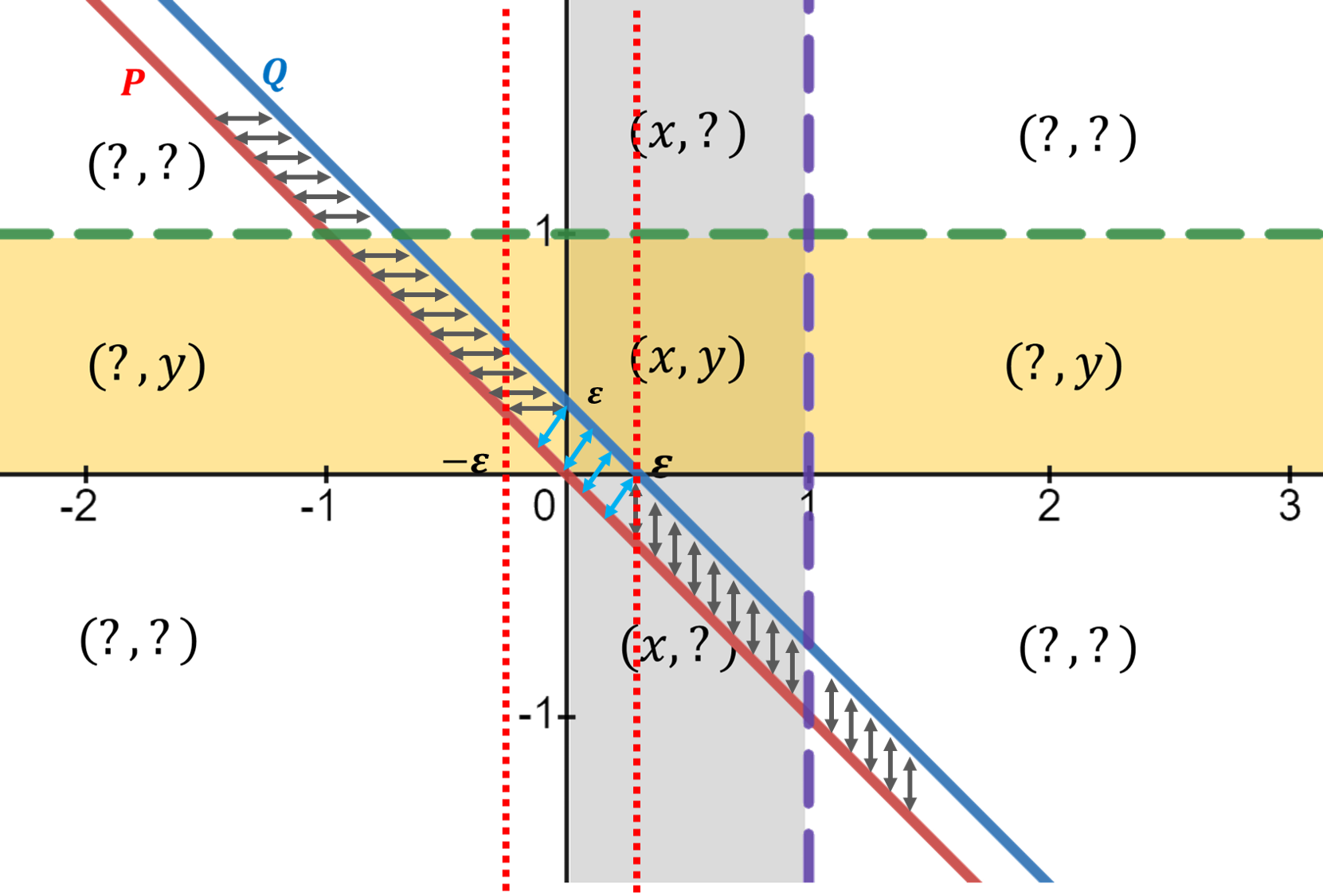}
    \caption{
In this example, the self-censoring missingness mechanism is as follows: Each coordinate $x,y$ of the sample is seen if and only if $x\in [0,1]$ or $y\in [0,1]$ respectively. }
\label{fig:2d_censoring}
\end{figure}
The censoring mechanism is as follows:
\begin{align*}
    &P^\obs: \mathcal{N}(\mathbf{0},\mathbf{I})\vert_{x + y = 0} \\
    &Q^\obs: \mathcal{N}(\mathbf{0},\mathbf{I})\vert_{x + y = \eps}
\end{align*}


We will now present the following lemma, which formally justifies the above remark. It also acts as a warm-up for our subsequent lower bound.

\begin{restatable}{lemma}{lemmalbzerovariance}
\label{lm:lb_zero_variance}
For any sufficiently small value of $\eps>0$, given $m=o(1/\eps)$ 
censored samples according to the missingness model $\obs$ from either $P^\obs$ or $Q^\obs$,  no algorithm can distinguish whether the samples are coming from $P^\obs$ or $Q^\obs$ with probability larger than $2/3$. 
\end{restatable}
\begin{proof}
We will define a coupling between the distributions $P^\obs$ and $Q^\obs$ such that with probability at least $1-O(\eps)$, the same missingness pattern and the same values for the seen coordinates are observed. This implies our lower bound of $\Omega(1/\eps)$ samples for distinguishing the two distributions with constant probability. 
 
More specifically, we define the following coupling matching the probability mass that points on the line $\{x+y=0\}$ have due to $P^\obs$ to the mass that $Q^\obs$ imposes on the line $\{x+y=\eps\}$:

\begin{equation}\label{eq:coupling}
(x,-x) \leftrightarrow \left\{
\begin{array}{lll}
      (x+\eps,-x) & x \leq -\eps  & \text{observed: $(?,-x)$}\\
      (x+\eps/2,-x+\eps/2) & -\eps/2\leq x\leq \eps/2 & \\
      (x,-x+\eps) &  x\geq \eps & \text{observed: $(x,?)$}\\
\end{array} 
\right. 
\end{equation}
For the segments of the support of $P^\obs$ unaccounted for in the above equation, we match their mass arbitrarily in a valid way to finish the coupling. Note that for the first and third branch of \cref{eq:coupling}, the point $(x,-x)$ has strictly larger probability density from $P^\obs$ than $(x+\eps,-x)$ and $(x,-x+\eps)$ respectively have from $Q^\obs$, while for any $x\in \mathbb{R}$, the point $(x+\eps/2,-x+\eps/2)$ has the same density in $Q^\obs$ as $(x,-x)$ has in $P^\obs$.

To see this, consider the coordinate system: $\left\{w:=\frac{x-y}{\sqrt{2}}, z:=\frac{x+y}{\sqrt{2}}\right\}$ (which is a rotation of the original one by $\pi/4$ rad). Since the distribution $\mathcal{N}(\mathbf{0},\mathbf{I})$ is isotropic, can now write the distributions $P^\obs$ and $Q^\obs$ as follows:  

\begin{align*}
    &P^\obs: \mathcal{N}(\mathbf{0},\mathbf{I})\vert_{z = 0}\\
    &Q^\obs: \mathcal{N}(\mathbf{0},\mathbf{I})\vert_{z = \eps/\sqrt{2}}
\end{align*}
Due to the fact that the marginals over $w$ and $z$ are independent, we have that $w\sim \gauss(0,1)$ in both of the cases above.

Given the above coupling, it follows that whenever the true sample has first coordinate in the set $(-\infty,-\eps)\cup (\eps,+\infty)$ the observed sample would be exactly the same (see \cref{eq:coupling}) either with $P^\obs$ or $Q^\obs$. 

Therefore, with probability at least $1-2\erf(\eps\sqrt{2})\geq 1-2\sqrt{2}\erf(\eps)$, the censored sample that we get will not give us any information for our task distinguishing $P^\obs$ from $Q^\obs$. Thus, any algorithm that is able to distinguish $P^\obs$ from $Q^\obs$ need to draw $\Omega(\frac{1}{\erf(\eps)})$ samples. This is $\Omega(1/\eps)$ samples as $\eps$ gets arbitrarily close to $0$. 
\end{proof}

\selfcensoringlb*

\begin{proof}

We now consider two  families of distributions parameterized by $\lambda\in [0,1]$. We define the two families in the rotated $(z,w)$ coordinate system, as in Lemma \ref{lm:lb_zero_variance}, as follows:
\begin{align}
\begin{split}
\label{eq:coupling_lambda}
    &P_\lambda: (z,w)\sim  \mathcal{N}\left(\mathbf{0},\begin{bmatrix}
\lambda & 0 \\
0 & 1 \\
\end{bmatrix}\right)\\
    &Q_\lambda: (z,w)\sim \mathcal{N}\left(\begin{bmatrix}
\eps  \\
0  \\
\end{bmatrix},\begin{bmatrix}
\lambda & 0 \\
0 & 1 \\
\end{bmatrix}\right)
\end{split}
\end{align}

Note that the total variation distance between $P_\lambda$ and $Q_\lambda$ is:
\small{ 
\begin{equation}
\begin{aligned}\label{eq:tv_dist}
    d_{TV}(P_\lambda,Q_\lambda)
    &= 
    d_{TV}(P_\lambda\vert_{w=0},Q_\lambda\vert_{w=0}) \\
    & =
    2\Phi\left(\frac{\eps}{2\sqrt{\lambda}}\right)-1=\erf\left(\frac{\eps}{2\sqrt{2}\sqrt{\lambda}}\right)=O\left(\frac{\eps}{\sqrt{\lambda}}\right).   
\end{aligned}
\end{equation} 
}
Consider the distributions $P_\lambda$ and $Q_\lambda$ defined in \cref{eq:coupling_lambda} for $\lambda=\lambda_{min}$. 
The main idea of the proof is that we can apply \cref{lm:lb_zero_variance} for the case where $\eps=\Theta(\sqrt{\lambda})$. 
Note that the distance between the means of these two distributions is $\eps/\sqrt{2}$. Thus, any algorithm that can estimate the mean with accuracy at most $\eps\sqrt{2}/4$, should be able to distinguish them. 

We will use the same coupling as in \cref{eq:coupling} between the two distributions $P_\lambda$ and $Q_\lambda$ and use it to bound the probability that we will observe different censored samples $p_c$ and $q_c$ respectively. We observe that (similarly to the setting of \cref{lm:lb_zero_variance}) any sample from $P_\lambda$ falling outside the band: $B=\{(x,y):-\eps\leq x\leq \eps\}$, has an identical censored sample to the censored sample of the corresponding point in $Q_\lambda$ via the coupling. 

We now upper bound the probability that sample from $P_\lambda$ falls in the band $B$:
\begin{equation}
 \Pr_{(x,y)\sim P_\lambda}\left[-\eps\leq x\leq \eps\right]\leq \Pr_{(z,w)\sim P_\lambda}\left[-\eps\sqrt{2}\leq w\leq \eps\sqrt{2}\right]=O(\eps)   
\end{equation}

Thus, we have: 
\begin{equation}\label{eq:coupling 1}
    \Pr[p_c\not = q_c]\leq \Pr_{(x,y)\sim P_\lambda}\left[-\eps\leq x\leq \eps\right] =O(\eps)
\end{equation}

In addition, due to \cref{eq:tv_dist}, there exists a different coupling for which the following holds:
\begin{equation}\label{eq:coupling 2}
    \Pr[p_c\not = q_c]\leq \Pr_{p\sim P_\lambda,q\sim Q_\lambda}[p\not = q] =O\left(\frac{\eps}{\sqrt{\lambda}}\right).
\end{equation}

By \cref{eq:coupling 1} and \cref{eq:coupling 2}, we get that no algorithm with $o(\max\{1/\eps, \sqrt{\lambda}/\eps\})$ samples can distinguish $P_\lambda$ from $Q_\lambda$ with probability at least $2/3$.

This implies the statement.

\end{proof}

\censoringmain*

\begin{proof}[Proof of Theorem \ref{th:censoring_main}]
    By Lemma \ref{lm:mean} and Lemma \ref{lm:sdp_spec}, we conclude that we can use $\tilde{O}(\frac{d^2}{\alpha\varepsilon^2})$ samples to get the following guarantees:
    \[
    \left\|\bm{\mu}-\bm{\hat{{\mu}}} \right\|_{2}\leq \varepsilon \sqrt{\lambda_{max}}
    \]
    and
    \[
    \| \mb{\Sigma^\ast} - \mb{\hat{\Sigma}}\|_F \leq \eps\lambda_{max}
    \]
    Thus,we get the following:
 \[
   \left\|\mb{\Sigma}^{\ast -1 / 2}(\bm{\mu}-\bm{\hat{\mu}}) \right\|_{2}\leq \frac{1}{\sqrt{\lambda_{min}}} \left\|\bm{\mu}-\bm{\hat{{\mu}}} \right\|_{2} \leq \varepsilon \sqrt{\lambda_{max}/\lambda_{min}}
    \]
\begin{align*}
    \left\|\bm{I}-\mb{\Sigma^{\ast -1 / 2}} \mb{\hat{\Sigma}} \mb{\Sigma^{\ast -1 / 2}}\right\|_{F}&=\left\|\mb{\Sigma^{\ast -1 / 2}} (\mb{\Sigma^\ast}-\mb{\hat{\Sigma}}) \mb{\Sigma^{\ast -1 / 2}}\right\|_{F}\\
    &\leq \frac{1}{\lambda_{min}} \| \mb{\Sigma^\ast} - \mb{\hat{\Sigma}}\|_F \\
    &\leq \eps\lambda_{max}/\lambda_{min}
\end{align*}
Thus, by substituting $\varepsilon^{\prime\prime}=\eps\lambda_{max}/\lambda_{min}$, we get sample complexity of $\tilde{O}(\frac{d^2(\lambda_{max}/\lambda_{min})^2}{\alpha\eps^{2}})$ for the following guarantees: 
 \[
   \left\|\mb{\Sigma}^{\ast -1 / 2}(\bm{\mu}-\bm{\hat{\mu}}) \right\|_{2}\leq \mathcal{O}(\varepsilon)
    \]
\[
    \left\|\bm{I}-\mb{\Sigma^{\ast -1 / 2}} \mb{\hat{\Sigma}} \mb{\Sigma^{\ast -1 / 2}}\right\|_{F}\leq \mathcal{O}(\eps)
\]
\end{proof}

\section{\cref{sec:convex-masking} OMITTED PROOFS}

\globalmin*
\begin{proof}
We first verify that the gradient vanishes at $\bm{\mu} = \bm{\mu}^*$. First, observe that 
\begin{align*}
\bm{\nabla}\ell(\bm{\mu}) &=
-\sum_{A\subseteq [d]} \int_{\mb{x} \in \R^{|A|}} \frac{\int_{\mb{y}}  \bm{\Sigma}^{-1} (\mb{y}-\bm{\mu})\cdot \indic[\obs(\mb{y}) = A] \cdot \delta(\mb{y}_A-\mb{x})  g_{\bm{\mu}}(\mb{y})  d\mb{y} }{g_{\bm{\mu}}^\obs(A, \mb{x})} g_{\bm{\mu}^*}^\obs(A, \mb{x}) d\mb{x}
\end{align*}
Hence:
\begin{align*}
\bm{\nabla}\ell(\bm{\mu}^*) &=
-\sum_{A\subseteq [d]} \int_{\mb{x} \in \R^{|A|}}{\int_{\mb{y}}  \bm{\Sigma}^{-1} (\mb{y}-\bm{\mu^*})\cdot \indic[\obs(\mb{y}) = A] \cdot \delta(\mb{y}_A-\mb{x}) g_{\bm{\mu^*}}(\mb{y})  d\mb{y} }d\mb{x}\\
&=- {\int_{\mb{y}}  \bm{\Sigma}^{-1} (\mb{y}-\bm{\mu^*})\cdot \sum_{A\subseteq [d]}\indic[\obs(\mb{y}) = A] \cdot \int_{\mb{x} \in \R^{|A|}}\delta(\mb{y}_A-\mb{x}) d\mb{x} \cdot g_{\bm{\mu^*}}(\mb{y})  d\mb{y} }\\
&=-{\int_{\mb{y}}  \bm{\Sigma}^{-1} (\mb{y}-\bm{\mu^*})\cdot g_{\bm{\mu^*}}(\mb{y})  d\mb{y} } = 0.
\end{align*}
One can also show this by using (\ref{eq:general_grad_exp}) for the gradient and then using the law of total expectation.
We next prove $\ell$ is convex by showing that $\bm{\nabla}^2\ell$ is positive semidefinite for any value of $\mb{\mu}$.  
\begin{align}
\begin{split}
&\bm{\nabla}^2 \ell(\bm{\mu})\\
&=\Exp_{(A,\mb{x})} \left[
-\frac{\int_{\mb{y}} (-\bm{\Sigma}^{-1} + \bm{\Sigma}^{-1}(\mb{y} - \bm{\mu})( \mb{y} - \bm{\mu})^T \bm{\Sigma}^{-1}) \indic[\obs(\mb{y}) = A] \cdot \delta(\mb{y}_A-\mb{x}) g_{\bm{\mu}}(\mb{y})  d\mb{y}}{g_{\bm{\mu}}^\obs(A,\mb{x})} \right]
+
\\
&\qquad
\Exp_{\mb{x} \sim \gauss(\bm{\mu}^*, \bm{\Sigma})} \left[
\left(\frac{\int_{\mb{y}}  \bm{\Sigma}^{-1}(\mb{y}-\bm{\mu})\indic[\obs(\mb{y}) = A] \cdot \delta(\mb{y}_A-\mb{x}) g_{\bm{\mu}}(\mb{y}) d\mb{y}}{\int_{\mb{y}} \indic[\obs(\mb{y}) = A] \cdot \delta(\mb{y}_A-\mb{x}) g_{\bm{\mu}}(\mb{y}) d\mb{y}}\right)^2
 \right] \\
 &= \Exp_{(A, \mb{x})}\left[\bm{\Sigma}^{-1} -\Cov_{\mb{y} \sim \gauss(\bm{\mu}, \bm{\Sigma})}\left[ \bm{\Sigma}^{-1} \left(\mb{y} - \bm{\mu}\right) \mid \obs(\mb{y})=A, \mb{y}_A = \mb{x}\right] \right].
\label{eq:general_hessian}
\end{split}
\end{align}
Observe that for a linear thresholding missingness pattern, the set $\{\mb{y}: \obs(\mb{y})=A, \mb{y}_A=\mb{x}\}$ is convex for any set $A$ and any $\mb{x} \in \R^{|A|}$. Using the fact (Corollary 2.1 of \cite{KP77}) that the variance of a Gaussian is non-increasing when restricted to a convex set:

\[
\Cov_{\mb{y} \sim \gauss(\bm{\mu}, \bm{\Sigma})}[\bm{\Sigma}^{-1}(\mb{y} - \bm{\mu}) \mid \obs(\mb{y})=A, \mb{y}_A = \mb{x}]  \preceq 
\Cov_{\mb{y} \sim \gauss(\bm{\mu}, \bm{\Sigma})}[\bm{\Sigma}^{-1}(\mb{y} - \bm{\mu})]  = \bm{\Sigma}^{-1}.
\]
Plugging into (\ref{eq:general_hessian}), we get that $\bm{\nabla}^2\ell(\bm{\mu}) \succeq 0$ for any $\bm{\mu}$. 
\end{proof}

\StrongConvexityMissingEntries*

\begin{proof}
Equivalently, we need to show that the minimum eigenvalue of the Hessian of the function $\ell(\bm{\mu})$ is at least $\lambda$. 
From (\ref{eq:general_hessian}), we have
\begin{equation} \label{hessian_l}
 \bm{\nabla}^2 \ell(\bm{\mu}) = \Exp_{(A, \mb{x})}\left[\bm{\Sigma}^{-1} -\Cov_{\mb{y} \sim \gauss(\bm{\mu}, \bm{\Sigma})}\left[ \bm{\Sigma}^{-1} \left(\mb{y} - \bm{\mu}\right) \mid \obs(\mb{y})=A, \mb{y}_A=\mb{x}\right] \right]
\end{equation}
Fix a unit vector $\mb{v} \in \R^d$. Let $H\subseteq [d]$ denote the set of indices that contains the $c=\beta d$ highest $v_i^2$ values. Therefore, $\sum_{i \in H} v_i^2 \geq \beta$. Using $\mb{v}$ as a test vector:
enecccfckgvkkjlfnlclcggguvjldtdurlrlltnecnid
\begin{align*}
\mb{v}^\top \nabla^2(\ell(\bm{\mu})) \mb{v} &= \mb{v}^\top \bm{\Sigma}^{-1}\mb{v} - \mb{v}^\top\cdot \left(\Exp_{(A,\mb{x})}\Cov_{\mb{y}}(\bm{\Sigma}^{-1}\mb{y} \mid\obs(\mb{y})=A, \mb{y}_A=\mb{x})\right)\cdot \mb{v}
\end{align*}
With probability $1-\alpha(H)$, some coordinate in $H$ is not in $A$. Under this event, we upper-bound $\Cov_{\mb{y}}(\bm{\Sigma}^{-1}\mb{y} \mid\obs(\mb{y})=A, \mb{y}_A=\mb{x})$ by $\Cov_{\mb{y}}(\bm{\Sigma}^{-1}\mb{y}) = \bm{\Sigma}^{-1}$ using the facts that the missingness is linear thresholding and the aforementioned Corollary 2.1 of \cite{KP77}. Under the complement event that $H \subseteq A$ is fully observed, we
use the upper-bound $\Cov_{\mb{y}}(\bm{\Sigma}^{-1}\mb{y} \mid \mb{y}_H=\mb{x}_H)$, again by the same facts.
By \cref{missingness_assumpt}, $\alpha(A) \geq \alpha$ and so:
\begin{align}
\mb{v}^\top \nabla^2(\ell(\bm{\mu})) \mb{v} 
 &\geq \alpha \left(\mb{v}^\top \bm{\Sigma}^{-1} \mb{v} - \mb{v}^\top\cdot \Cov_{\mb{y}}(\bm{\Sigma}^{-1}\mb{y} \mid \mb{y}_H =  \mb{x}_H)\cdot\mb{v} \right)\nonumber\\
 &= \alpha \mb{v}^\top \left(I - \bm{\Sigma}^{-1}\Var_{\mb{y}}(\mb{y} \mid \mb{y}_H = \mb{x}_H)\right) \bm{\Sigma}^{-1}\mb{v}\label{eq:hesslb}
\end{align}
We use the standard facts that for any set $A$: 
\begin{enumerate}
\item[(i)]
$\Var_{\mb{y}}(\mb{y} \mid \mb{y}_A = \mb{x}_A) = \begin{bmatrix} 0 & 0 \\ 0 & \bm{\Sigma}_{\bar{A}\bar{A}}-\bm{\Sigma}_{\bar{A}A}\bm{\Sigma}_{AA}^{-1}\bm{\Sigma}_{A\bar{A}}\end{bmatrix}$
 \item[(ii)]
 $\bm{\Sigma}^{-1} = \begin{bmatrix} \bm{\Sigma}_{AA}^{-1} + \bm{\Sigma}_{AA}^{-1}\bm{\Sigma}_{A\bar{A}} (\bm{\Sigma}_{\bar{A}\bar{A}} - \bm{\Sigma}_{\bar{A}A}\bm{\Sigma}_{AA}^{-1}\bm{\Sigma}_{A\bar{A}})^{-1}\bm{\Sigma}_{\bar{A}A}\bm{\Sigma}_{AA}^{-1} & -
 \bm{\Sigma}_{AA}^{-1}\bm{\Sigma}_{A\bar{A}} (\bm{\Sigma}_{\bar{A}\bar{A}} - \bm{\Sigma}_{\bar{A}A}\bm{\Sigma}_{AA}^{-1}\bm{\Sigma}_{A\bar{A}})^{-1}\\ -(\bm{\Sigma}_{\bar{A}\bar{A}} - \bm{\Sigma}_{\bar{A}A}\bm{\Sigma}_{AA}^{-1}\bm{\Sigma}_{A\bar{A}})^{-1}\bm{\Sigma}_{\bar{A}A}\bm{\Sigma}_{AA}^{-1} &(\bm{\Sigma}_{\bar{A}\bar{A}} - \bm{\Sigma}_{\bar{A}A}\bm{\Sigma}_{AA}^{-1}\bm{\Sigma}_{A\bar{A}})^{-1} \end{bmatrix}$
 \end{enumerate}
 Using them repeatedly to simplify (\ref{eq:hesslb}), we get:
 \begin{align*}
\mb{v}^\top \nabla^2(\ell(\bm{\mu})) \mb{v}
&\geq \alpha \mb{v}^\top \left(I - \begin{bmatrix}0 & -\bm{\Sigma}_{HH}^{-1}\bm{\Sigma}_{H\bar{H}}\\0 & I\end{bmatrix}\right) \bm{\Sigma}^{-1}\mb{v}\\
&= \alpha \mb{v}^\top \begin{bmatrix} I & \bm{\Sigma}_{HH}^{-1}\bm{\Sigma}_{H\bar{H}}\\0 & 0\end{bmatrix} \bm{\Sigma}^{-1}\mb{v}\\
&= \alpha \mb{v}^\top \begin{bmatrix}\bm{\Sigma}_{HH}^{-1} & 0\\0 & 0\end{bmatrix} \mb{v}\\
&= \alpha \mb{v}_H^\top \bm{\Sigma}_{HH}^{-1} \mb{v}_H
 \end{align*}
Finally, we use our choice of $H$ and that $\lambda_{\min}(\bm{\Sigma}_{HH}^{-1}) = 1/\lambda_{\max}(\bm{\Sigma}_{HH}) \geq 1/\lambda_{\max}(\bm{\Sigma})$ to obtain our claim.
\end{proof}
We now describe our solutions to the following three problems as outlined in \cref{sec:convex-masking}.

\begin{itemize}[leftmargin=*]
    \item \textbf{Initialization}: efficiently compute an initial feasible point from which to start the optimization. The pseudocode for \texttt{Initialize} appears in \cref{algo:init};
    \item \textbf{Gradient estimation}: design a nearly unbiased sampler for $\bm{\nabla}\ell(\bm{\mu})$ using Langevin sampling. The \texttt{SampleGradient} pseudocode appears in \cref{algo:sampgrad}; 
    \item \textbf{Efficient projection}: perform an efficient projection into a set of feasible points to make sure that PSGD converges. The pseudocode presents in \cref{algo:projettodomain}. 
\end{itemize}

\subsection{GRADIENT ESTIMATION}
Recall from the gradient expression in (\ref{eq:general_grad}) that the main 
obstacle in computing $\nabla \ell(\bm{\mu})$  is the term $\Exp_{(A,\mb{x})}\Exp_{\mb{y}}[\mb{y} \mid \obs(\mb{y})=A, \mb{y}_A=\mb{x}]$. Here, $(A,\mb{x})$ is an observation generated from $\gauss(\bm{\mu}^*, \bm{\Sigma})^\obs$, while $\mb{y}$ is sampled from $\gauss(\bm{\mu}, \bm{\Sigma})$ for some $\bm{\mu} \in \R^d$.  So, to implement \texttt{SampleGradient}, we need an approximately unbiased estimator.

The most straightforward way is to apply rejection sampling: for an $(A, \mb{x})$ generated by $\mathcal{O}$, keep sampling $\mb{y}$ from the conditional\footnote{We need to sample from the conditional distribution because $\mb{y}\equiv \mb{x}$ occurs with probability $0$ in the $\gauss(\bm{\mu},\bm{\Sigma})$ measure.} distribution $\gauss(\bm{\mu},\bm{\Sigma}) \mid \mb{y}_A = \mb{x})$ until $\obs(\mb{y}) = A$. If $\bm{\mu}^* = \bm{\mu}$, then the expected cost of the rejection sampling is $O(1/\gamma)$ by \cref{dependence_assumpt}, as $\mb{y}_A = \mb{x}$ implies $\mb{y}_C = \mb{x}_C$. The issue that arises is that the probability of $\obs(\mb{y}) = A$ can decrease exponentially in the distance between $\bm{\mu}$ and $\bm{\mu}^*$, and so, rejection sampling becomes infeasible.
\ignore{To understand the issue more formally, we define $\gamma(\mb{x})$ to be the probability that a sample from the conditional distribution is not rejected when $\bm{\mu}=\bm{\mu}^*$:
\begin{equation}
\gamma(\mb{x}) \coloneqq \Pr_{\mb{y} \sim \gauss(\bm{\mu}^*, \bm{\Sigma})}\left[\cA(\mb{y}) = \cA(\mb{x}) \mid y_i = x_i \text{ for all } i \in \cA(\mb{x})\right]
\end{equation}}

To sample the gradient when $\bm{\mu}$ is far from $\bm{\mu}^*$, we use the projected Langevin Monte Carlo algorithm \cite{bubeck2018sampling}. For an observation $(A, \mb{x})$, suppose $\mb{y} \sim \gauss(\bm{\mu}, \bm{\Sigma})$, and let $\bm{\mu}_{\text{cond}}$ and ${\bm{\Sigma}}_{\text{cond}}$ be the mean and covariance of $\mb{y}_{\bar{A}}$ conditioned on $\mb{y}_A = \mb{x}$. It is well-known that the conditional distribution is Gaussian with:
\begin{align}\label{eq:newg}
\bm{\mu}_{\text{cond}}  &= \bm{\mu}_{\bar{A}} + \bm{\Sigma}_{\bar{A}, A} \bm{\Sigma}_{A, A}^{-1} (\mb{{x}} - \bm{\mu}_A)\\
\bm{\Sigma}_{\text{cond}} &= \bm{\Sigma}_{\bar{A}, \bar{A}} - \bm{\Sigma}_{\bar{A}, A} \bm{\Sigma}_{A,A}^{-1}\bm{\Sigma}_{A, \bar{A}}
\label{eq:newc}
\end{align}
    where ${\bar{A}}$ represents $[d] \setminus A$. Let\footnote{$\mb{x}_A \circ \mb{z}$, where $\mb{z} \in \R^{d-|A|}$, denotes the vector $\mb{y} \in \R^d$ where $\mb{y}_A = \mb{x}_A$ and $\mb{y}_{\bar{A}} = \mb{z}$.} $\cK = \{\mb{z} \in \R^{d-|A|} \mid \obs(\mb{x} \circ \mb{z}) = A\}$. 
The iteration of the projected Langevin Monte Carlo algorithm takes the following form:

\begin{equation}
\mb{z}^{(t+1)} = \Pi_{\cK \cap \ball_{\textcolor{black}{\Sigma}}(\bm{\mu},R)}\left(\mb{z}^{(t)} - \frac{\eta}{2}\bm{\Sigma}_{\text{cond}}^{-1}(\mb{z}^{(t)} - \bm{\mu}_{\text{cond}}) + \sqrt{\eta} \cdot \bm{\zeta}^{(t)}\right)
\end{equation}
where $\eta$ is a step-size parameter,
$R$ is an appropriate radius parameter, and $\bm{\zeta}^{(0)}, \bm{\zeta}^{(1)}, \dots$ are i.i.d.~samples from the standard normal distribution in $(d-|A|)$-dimensions. 
We implicitly make the reasonable assumption here that 
Mahalanobis projection 
to the convex set 
$\Pi_{\cK \cap \ball_{\Sigma}(\bm{\mu},R)}(\cdot)$ can be performed efficiently. The pseudocode for \texttt{SampleGradient} appears in \cref{algo:sampgrad}.
\ignore{
{
\IncMargin{1.2em}
\begin{algorithm}
\SetKwInOut{Input}{Input}
\Input{$(A, \mb{x})$, $\bm{\mu}$, parameters $\eta, R, M$}
$a \gets |A|$\\
Compute $\bm{\mu}_{\text{cond}}$ and $\bm{\Sigma}_{\text{cond}}$ as in (\ref{eq:newg}) and (\ref{eq:newc})\\
Let $W$ be such that $\bm{\Sigma}_{\text{cond}} = WW^\top$\\
Compute $\cL = (W^{-1} \cK)\cap \ball_{\textcolor{black}{\Sigma}}(W^{-1}\bm{\mu}_{\text{cond}},R)$ \\
$\mb{z}^{(0)} \gets $ a point in $\cL$\\
\For{$t = 0~\KwTo~ M-1$}{
Sample $\bm{\zeta}^{(t)}$ from $\gauss(0, I_{d-a})$\\
$\mb{z}^{(t+1)} \gets \Pi_\cL\left(\mb{z}^{(t)} - {\eta}(\mb{z}^{(t)} - W^{-1}\bm{\mu}_{\text{cond}}) + \sqrt{\eta} \cdot \bm{\zeta}^{(t)}\right)$
}
\Return $-\bm{\Sigma}^{-1}(\mb{x} \circ (W\mb{z}^{(M)})-\bm{\mu})$ 
\caption{[\texttt{SampleGradient}] Sampler for $\nabla \ell(\bm{\mu})$.}
\label{algo:sampgrad}
\end{algorithm}
\DecMargin{1.2em}
}
}

\subsection{PROJECTION TO FEASIBLE DOMAIN}
Next, in each iteration of SGD in \texttt{MissingDescent} (\cref{algo:psgd}), 
we need to choose a projection set to make sure that PSGD converges. Specifically, we project a current guess back to a $\ball_{\bm{\Sigma}}$ ball centered at $\bm{\mu}^{(0)}$.

\subsection{BOUNDED STEP VARIANCE AND GRADIENT BIAS}

\begin{restatable}{lemma}{BoundedStepVariance}
\label{lem:incirc}
$\cK \cap \ball_{\bm{\Sigma}}(\bm{0}, R_1) \supseteq \ball_{\bm{\Sigma}}(\mb{c}, r)$ for some $\mb{c}$, where $R_1 =  \sqrt{d} + O(\sqrt{\log(1/\gamma)}$ and $r = \Omega(\gamma/d^2)$
\end{restatable}

\begin{restatable}{lemma}{circumcirc}
\label{lem:circumcirc}
For $\mb{w}\sim \gauss(\bm{\mu}, \bm{\Sigma})$, where $\|\bm{\mu}\|_{\bm{\Sigma}}\leq S$, 
\[
\Pr[\mb{w} \notin \ball_{\bm{\Sigma}}(\bm{\mu}, R_2)\cap K] \leq \frac{\eps}{16}\cdot \Pr[\mb{w} \in K]
\]
where $R_2 = \poly{d, S, \log(1/\gamma), \log(1/\eps) } > R_1$.
\end{restatable}
\begin{proof}
Suppose $\mb{v} \sim \gauss(\bm{0}, \bm{\Sigma})$.  By standard concentration of gaussians:
\[
\Pr[\|\mb{v}\|_{\bm{\Sigma}} > \sqrt{d} + O(\sqrt{\log(1/\delta)})]\leq \delta
 \]
Setting $\delta = \gamma/2$, we get that:
\begin{equation}\label{eq:truncK2}
\Pr[\mb{v} \in \ball_{\bm{\Sigma}}(\bm{0}, R_1) \cap \cK] \geq \gamma/2.
\end{equation}
\ignore{
This implies that\arnab{Not sure. Check this.}:
\[
\forall \|\mb{u}\|=1:~~\max_{\mb{x}, \mb{y} \in \ball(\bm{0}, R_1) \cap K} \mb{u}^\top (\mb{x}-\mb{y}) \geq \Omega(\gamma).
\]}
 Invoking Lemma 12 of \cite{cherapanamjeri2022makes}, we get that there exists $\mb{c}$ such that $\ball_{\bm{\Sigma}}(\bm{0}, R_1) \cap \cK$ contains $\ball_{\bm{\Sigma}}(\mb{c}, r)$ for $r=\Omega(\gamma/d^2 )$.
\end{proof}

\circumcirc*
\begin{proof}
Using (\ref{eq:truncK2}):
\begin{align*}
\Pr[\mb{w}\in \cK] \geq \Pr[\mb{w} \in \ball_{\bm{\Sigma}}(\bm{0}, R_1) \cap \cK] 
&\geq \frac{\gamma}{2}\exp\left(-\frac{\Vert\bm{\mu}\Vert_{\bm{\Sigma}}^2 + 2R_1}{2}\right) 
\geq \frac{\gamma}{2}\exp\left( -\frac{S^2}{2} - R_1 \right)
\end{align*}
Call the lower-bound on the right $\gamma'$. Note that $\gamma'$ may be exponentially smaller than $\gamma$ for large $S$. 

We define $R_2$ large enough so that $\Pr[\mb{w} \notin \ball_{\bm{\Sigma}}(\bm{\mu}, R_2)]\leq \frac{\eps}{16}\gamma'$. By
concentration of gaussians, it suffices to take $R_2 = \sqrt{d} + O(\sqrt{\log(1/\gamma')}) = \sqrt{d} + O\left(\sqrt{\log\frac{1}{\gamma\eps} + S^2+R_1}\right)$.   Note that the claim about $R_2$ follows.
\end{proof}

\begin{figure*}
    \centering
    \includegraphics[width=0.7\textwidth]{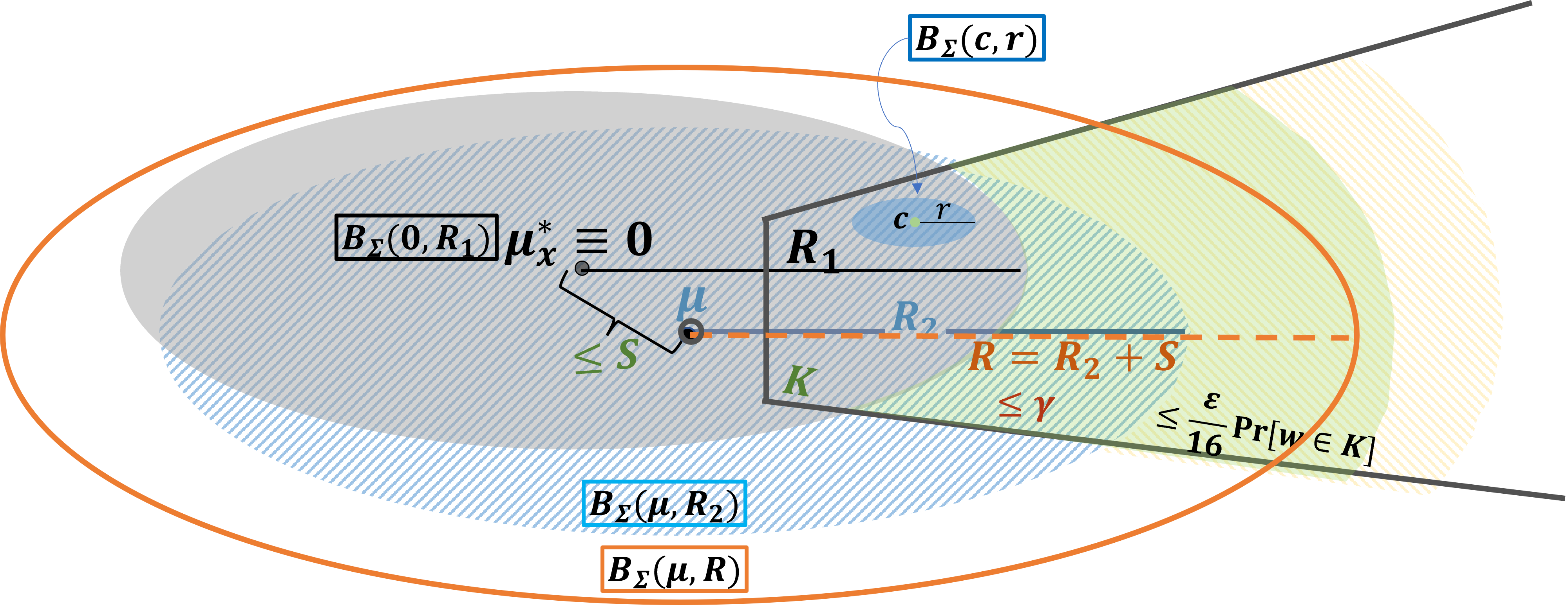}
    \caption{
An illustration of convex sets in \cref{sec:bounded_variance_and_bias}.}
\label{fig:convex_sets}
\end{figure*}

\langevin*

\begin{proof}[Proof of \cref{Thm:langevin}]
With $R_2$ as in \cref{lem:circumcirc}, set $R=R_2+S$, so that $\cL = \ball_{\bm{\Sigma}}(\bm{\mu}, R) \cap \cK$. Since $R_2 > R_1$, \cref{lem:incirc} implies that $\cL$ contains a ball of radius $r$. On the other hand, by \cref{lem:circumcirc}, $\Pr[\mb{w} \notin \cL] \leq \frac{\eps}{4} \Pr[\mb{w} \in \cK]$ for $\mb{w} \sim \gauss(\bm{\mu}, \bm{\Sigma})$, which implies that the truncation of $\gauss(\bm{\mu}, \bm{\Sigma})$ to $\cK$ and to $\cL$ are at most $\eps/2$ far from each other in TV distance.

We can now use the main result of \cite{bubeck2018sampling} to approximately sample from the truncated gaussian $\gauss(\bm{\mu},\bm{\Sigma}; \cL)$ with TV error $\eps/2$. This work analyzes the projected Langevin Monte Carlo algorithm for sampling from a distribution $\mu$ on $\R^d$ whose density is proportional to $\exp(-f(\mb{x}))\cdot 1[\bm{x} \in \mathcal{M}]$ where $\mathcal{M}$ is a convex body containing the origin.  Suppose for all $\mb{x}, \mb{y} \in \mathcal{M}$, 
$\|\nabla f(\mb{x})-\nabla f(\mb{y})\| \leq \beta \|\mb{x}-\mb{y}\|$  and $\|\nabla f(\mb{x})\| \leq \ell$. Consider the Langevin dynamics with $\bar{\mb{X}}_0 = \bm{0}$ and:
\[
\bar{\mb{X}}_{k+1} = \Pi_{\mathcal{M}}\left(\bar{\mb{X}}_k - \frac{\eta}{2} \nabla f(\bar{\mb{X}}_k) + \sqrt{\eta} \bm{\zeta}_k\right)
\]
where $\bm{\zeta}_0, \bm{\zeta}_1, \dots$ are i.i.d.~standard normal variables.
If $\mathcal{M}$ contains a Euclidean ball of radius $1$ and is contained in a Euclidean ball of radius $R_{\text{out}}$, then Theorem 1 of \cite{bubeck2018sampling} claims that $\tv(\bar{\mb{X}}_N, \mu)\leq \eps$ if $\eta = \tilde{\Theta}(R_{\text{out}}^2/N)$ and $N = \tilde{\Omega}(R_{\text{out}}^6 \max(d, R_{\text{out}}\ell, R_{\text{out}}\beta)^{12}/\eps^{12})$. 

In our context, \cref{algo:sampgrad} already transforms by a Cholesky decomposition of $\bm{\Sigma}$ so as to transform Mahalanobis distance to Euclidean distance. We can then scale by $r$ (from \cref{lem:incirc}) to ensure that a Euclidean unit ball is contained inside the transformed $\cL$. The radius of the outer ball is then $R_{\text{out}} \leq R/r \leq \tilde{O}(d^3 S/\gamma)$. We can bound the parameters $\beta$ and $\ell$ (similarly to Section B.3 of \cite{cherapanamjeri2022makes}):
\[
\beta = O\left(\frac{\gamma^2}{d^4}\right) \qquad \qquad \ell = \tilde{O}\left(\frac{\gamma S}{d^{1.5}}\right)
\]
So, invoking the result of \cite{bubeck2018sampling}, for any particular $\mb{x}$, the running time of \texttt{SampleGradient} is $\poly{d, S, 1/\gamma, 1/\eps}$. 
\end{proof}

\boundedvariance*

\begin{proof}
We first show \eqref{expectation_approx}:
\begin{align*}
\|\Exp[\hat{\mb{g}}]-\mb{g}\|_2
&= \left\|\bm{\Sigma}^{-1}\left(\mb{{x}}_A\circ \Exp[W\mb{z}^{(M)}] - \mb{\tilde{x}}_A \circ \Exp[\mb{y}_{\bar{A}} \mid \obs(\mb{y})=A, \mb{y}_A=\mb{x}]\right)\right\|\\
&\leq \frac{1}{\lambda_{\min}(\bm{\Sigma})} \|\Exp[W\mb{z}^{(M)}] - \Exp[\mb{y}_{\bar{A}} \mid \obs(\mb{y})=A, \mb{y}_A=\mb{x}]\|\\
&\leq \frac{\eps\sqrt{\lambda_{\max}(\bm{\Sigma})}}{\lambda_{\min}(\bm{\Sigma})}(R + O(\sqrt{\log(1/\gamma)})) 
\end{align*}

The last inequality holds by the guarantee of \cref{Thm:langevin}, as well as the fact that $\mb{z}^{(M)}$ is contained within $\ball_{\bm{\Sigma}}(\mb{\mu}_{\text{cond}}, R)$ while $\Exp[\mb{y}_{\bar{A}} \mid \mb{y}_A = {\mb{x}}_A, \cA(\mb{y})=A]$ is within $\ball_{\bm{\Sigma}}(\bm{\mu}_{\text{cond}}, O(\sqrt{\log(1/\gamma))})$ by \cref{dependence_assumpt} and Lemma 6 of \cite{daskalakis2018efficient}.

Given the above and the existence of the projection step in the gradient estimator, we can get the following bound on the centralized second moment of the gradient estimator:

\[
\mathbb{E}[\Vert\hat{\mb{g}}\Vert_2^2]
\leq \Exp[\|\bm{\Sigma}^{-1}(\mb{y}-\bm{\mu})\|^2 \mid \obs(\mb{y})=A, \mb{y}_A=\mb{x}] + \eps\left(\frac{R} {\sqrt{\lambda_{\min}}}\right)^2
\leq \poly{d,1/\gamma,S,{1}/{\lambda_{\min}}}
\]
\end{proof}

\subsubsection{Bound on the Initialization}

\begin{restatable}[Empirical Parameters vs True Parameters]{lemma}{InitializationSampleComplexity}
\label{lm:initialization sample complexity}
The empirical mean $\mathbb{E}[w]$
computed using $\tilde{\mathcal{O}}(\frac{ d \log( n d/\alpha\beta \delta) \log(1 / \delta \beta)}{\varepsilon^2})$ samples by sampling from the general missingness model with probability at least $1 - \delta$
satisfy  $\|\mb{w}-\bm{\mu}^*\|_2 \leq \mathcal{O}\left({\sqrt{\frac{\lambda_{\max}}{\beta} \log(1/\alpha)}}\right)$.
\end{restatable}
\begin{proof}
For each iteration of the for loop in \cref{algo:init} uses $\tilde{\mathcal{O}}(\frac{ \beta d \log( n d/\alpha \delta^\prime) \log(1 / \delta^\prime)}{\varepsilon^2})$ samples due to the sample complexity bound in Lemma 5 of \cite{cherapanamjeri2022makes} by applying Lemma 5 on Lemma 6 (1) using triangle inequality,
and the bound holds with probability $1 - \delta^\prime$. Since the for loop makes \textcolor{black}{$\lceil\frac{1}{\beta}\rceil$} iterations, we conclude that using $\tilde{\mathcal{O}}(\frac{ d \log( n d/\alpha \delta^\prime) \log(1 / \delta^\prime)}{\varepsilon^2})$ samples, our algorithm satisfies with probability at least \textcolor{black}{$1 - \delta^\prime\cdot\lceil\frac{1}{\beta}\rceil$} using union bound. 
Let \textcolor{black}{$\delta = \delta^\prime\cdot\lceil\frac{1}{\beta}\rceil$.}
We have $\delta^\prime = \textcolor{black}{O(\delta \beta)}$. Therefore, with $\tilde{\mathcal{O}}(\frac{ d \log( n d/\alpha\beta \delta) \log(1 / \delta \beta)}{\varepsilon^2})$ samples, with probability at least $1 - \delta$, the output of \cref{algo:init} satisfies that $\|\mb{w}-\bm{\mu}^*\|_2 \leq \mathcal{O}\left({\sqrt{\frac{\lambda_{\max}}{\beta} \log(1/\alpha)}}\right)$.
\end{proof}
\vfill

\ignore{
\begin{restatable}{lemma}{BoundedStepVariance}
\label{lem:incirc}
$\cK \cap \ball_{\bm{\Sigma}}(\bm{0}, R_1) \supseteq \ball_{\bm{\Sigma}}(\mb{c}, r)$ for some $\mb{c}$, where $R_1 =  \sqrt{d} + O(\sqrt{\log(1/\gamma)}$ and $r = \Omega(\gamma/d^2)$
\end{restatable}

\begin{restatable}{lemma}{circumcirc}
\label{lem:circumcirc}
For $\mb{w}\sim \gauss(\bm{\mu}, \bm{\Sigma})$, where $\|\bm{\mu}\|_{\bm{\Sigma}}\leq S$, 
\[
\Pr[\mb{w} \notin \ball_{\bm{\Sigma}}(\bm{\mu}, R_2)\cap K] \leq \frac{\eps}{16}\cdot \Pr[\mb{w} \in K]
\]
where $R_2 = \poly{d, S, \log(1/\gamma), \log(1/\eps) } > R_1$.
\end{restatable}

The proof of our main theorem in this section is given below:
}

\end{document}